\documentclass[preprint,12pt]{elsarticle}

\usepackage{lineno,hyperref}
\usepackage{pifont}
\modulolinenumbers[5]

\usepackage{colortbl}
\usepackage{soul}
\soulregister{\cite}7 
\soulregister{\citep}7 
\soulregister{\citet}7 
\soulregister{\ref}7 
\soulregister{\pageref}7 
\usepackage{graphicx}
\usepackage{amsmath}
\usepackage{amssymb}
\usepackage{booktabs}
\usepackage{algorithm}
\usepackage{algorithmicx}
\usepackage{algpseudocode}
\usepackage{color,xcolor}
\usepackage{multirow}
\usepackage{ntheorem}
\newtheorem{definition}{Definition}
\newtheorem{theorem}{Theorem}
\newtheorem*{proof}{Proof}
\usepackage{pifont}
\usepackage{subcaption}

\journal{Journal of \LaTeX\ Templates}

\begin{document}

\begin{frontmatter}



\title{Improving the Transferability of Adversarial Examples via Direction Tuning}

\author[mymainaddress]{Xiangyuan Yang}\ead{ouyang\_xy@stu.xjtu.edu.cn}
\author[mymainaddress]{Jie Lin\corref{mycorrespondingauthor}}
\cortext[mycorrespondingauthor]{Corresponding author}\ead{jielin@mail.xjtu.edu.cn}
\author[mythirdaddress]{Hanlin Zhang}\ead{hanlin@qdu.edu.cn}
\author[mymainaddress]{Xinyu Yang}\ead{yxyphd@mail.xjtu.edu.cn}
\author[mymainaddress]{Peng Zhao}\ead{p.zhao@mail.xjtu.edu.cn}

\address[mymainaddress]{School of Computer Science and Technology, Xi'an Jiaotong University, Xi'an, China}
\address[mythirdaddress]{Qingdao University, Qingdao, China}

\begin{abstract}
In the transfer-based adversarial attacks, adversarial examples are only generated by the surrogate models and achieve effective perturbation in the victim models.  Although considerable efforts have been developed on improving the transferability of adversarial examples generated by transfer-based adversarial attacks, our investigation found that, the big convergence speed deviation along the actual and steepest update directions of the current transfer-based adversarial attacks is caused by the large update step length, resulting in the generated adversarial examples can not converge well. However, directly reducing the update step length will lead to serious update oscillation so that the generated adversarial examples also can not achieve great transferability to the victim models. To address these issues, a novel transfer-based attack, namely direction tuning attack, is proposed to not only decrease the angle between the actual and steepest update directions in the large step length, but also mitigate the update oscillation in the small sampling step length, thereby making the generated adversarial examples converge well to achieve great transferability on victim models. Specifically, in our direction tuning attack, we first update the adversarial examples using the large step length, which is aligned with the previous transfer-based attacks. In addition, in each large step length of the update, multiple examples sampled by using a small step length. The average gradient of these sampled examples is then used to reduce the angle between the actual and steepest update directions, as well as mitigates the update oscillation by eliminating the oscillating component. By doing so, our direction tuning attack can achieve better convergence and enhance the transferability of the generated adversarial examples. In addition, a network pruning method is proposed to smooth the decision boundary, thereby further decreasing the update oscillation and enhancing the transferability of the generated adversarial examples. The experiment results on ImageNet demonstrate that the average attack success rate (ASR) of the adversarial examples generated by our method can be improved from 87.9\% to 94.5\% on five victim models without defenses, and from 69.1\% to 76.2\% on eight advanced defense methods, in comparison with that of latest gradient-based attacks. Code is available at https://github.com/HaloMoto/DirectionTuningAttack.
\end{abstract}




\end{frontmatter}


%
\begin{figure}[ht]
\begin{center}
\centerline{\includegraphics[width=\columnwidth]{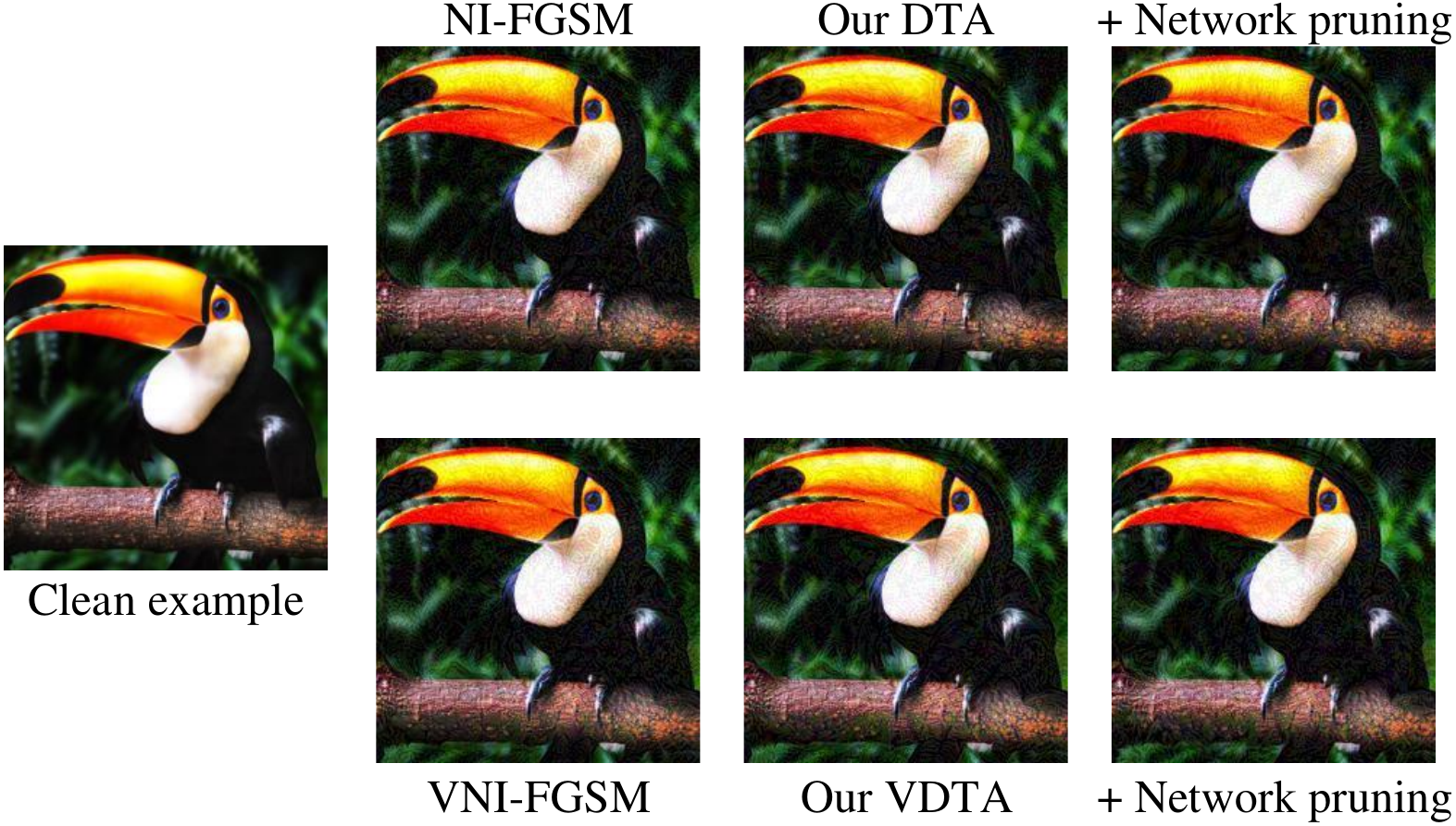}}
\caption{The comparison between the latest gradient-based attacks and our methods. The adversarial examples are generated by ResNet50~\cite{ResNet}. The fourth column denotes the adversarial examples generated by the combination of direction tuning attacks and network pruning method.}
\label{fig:adversarial_example_visualization}
\end{center}
\end{figure}
%

\section{Introduction}
\label{sec:introduction}
Due to the linear property~\cite{Intriguing-properties} of deep neural networks (DNN)~\cite{ResNet,VGG,Inception-v3}, the deep learning model is sensitive to unseen perturbation~\cite{CW,Madry-AT,FGSM,I-FGSM,MI-FGSM,SI-NI-FGSM,VMI-FGSM,Square,ZOO}, which seriously threatens the critical DNN-based applications such as self-driving technique and healthcare. The transfer-based adversarial attacks, taking advantage of generating adversarial examples without requiring any knowledge of the victim models,  have been widely developed to promote the robustness of the training models on deep neural networks, thereby increasing the reliability of the DNN-based applications.

Presently, a number of efforts on transfer-based adversarial attacks have been developed to focus on improving the transferability of generated adversarial examples to the victim models, because the greater the transferability, the higher the attack success rate on the victim model. The latest transfer-based attacks can be categorized as gradient-based attacks~\cite{FGSM,I-FGSM,MI-FGSM,SI-NI-FGSM,VMI-FGSM}, input transformation methods~\cite{SI-NI-FGSM,DI-FGSM,TI-FGSM,Admix}, feature importance-aware attacks~\cite{FIA,NAA}, model-specific attacks~\cite{SGM,LinBP} and model diversification methods~\cite{LGV,MoreBayesian}. Due to the latter four categories of attacks being developed based on gradient-based attacks, the investigation of gradient-based attacks is essential for enhancing the efficiency of transfer-based attacks.

In the current gradient-based attacks, Kurakin {\em et al.}~\cite{I-FGSM} first proposed the iterative-based fast gradient sign method (I-FGSM). Dong {\em et al.}~\cite{MI-FGSM} introduced momentum into I-FGSM and proposed  MI-FGSM to accelerate the convergence of generating adversarial examples. Based on MI-FGSM, Lin {\em et al.}~\cite{SI-NI-FGSM} introduced the looking ahead operation (NI-FGSM) to further speed up the convergence. Wang {\em et al.}~\cite{VMI-FGSM} proposed variance tuning to stabilize the update direction of MI-FGSM and NI-FGSM. However, we found that, the update step length of these attacks is large, leading to a big convergence speed deviation between the actual update direction and the steepest update direction, which will be verified by Theorem~\ref{theorem1} in our paper. In this scenario, the generated adversarial examples can not converge well to achieve great perturbation on the victim models. That is, the generated adversarial examples achieve poor transferability.

We also found that, directly reducing the step length will lead to serious update oscillation, resulting in the reduction of the transferability of adversarial examples generated by MI-FGSN~\cite{MI-FGSM} and NI-FGSM~\cite{SI-NI-FGSM}, which will be verified in Fig.~\ref{fig:different_eps_steps}. In addition, the update oscillation also has a negative effect on further improving the transferability of adversarial examples generated by VMI-FGSM~\cite{VMI-FGSM} and VNI-FGSM~\cite{VMI-FGSM} attacks.

To address these issues, we proposed a novel transfer-based adversarial attack, namely direction tuning attack, in which the small sampling step length is embedded into the large update step length to not only decrease the angle between the actual direction and the steepest direction, but also mitigate the update oscillation. Particularly, the adversarial example in our direction tuning attack is updated by using the large step length that aligns with the previous gradient-based attacks~\cite{I-FGSM,MI-FGSM,SI-NI-FGSM,VMI-FGSM}. Then, the several examples are sampled with a small step length, and the average gradient of these sampled examples is used as the update direction in each large step length of the update. Through the coordinated use of the large update step length and the small sampling step length, our direction tuning attack can effectively reduce the angle between the actual direction and the steepest direction, while simultaneously eliminating the oscillating component that arises during the optimization process. As a result, the adversarial examples generated by our method are able to converge efficiently and achieve high transferability to a variety of victim models. In addition, the network pruning method is proposed in our paper to smooth the decision boundary by pruning the redundant or unimportant neurons to further eliminate update oscillation for further enhancing the transferability of adversarial examples. Fig.~\ref{fig:adversarial_example_visualization} shows that the adversarial examples generated by our methods (i.e., direction tuning attacks with network pruning) are imperceptible by human eyes, which are similar to the latest gradient-based attacks~\cite{SI-NI-FGSM,VMI-FGSM}.

Our main contributions are summarized as follows:
\begin{itemize}
\item First, we theoretically found the large step length of the update will lead to a big convergence speed deviation along the actual update direction and the steepest update direction in adversarial examples generation. Moreover, we empirically found that directly reducing the step length of the update will lead to serious update oscillation and low transferability of the generated adversarial examples to the victim models as well. Therefore, this leads to contemplating how to utilize the small length of steps to enhance the transferability.

\item Second, direction tuning attack is proposed to utilize the small length of step to enhance the transferability, in which an inner loop sampling module is added in each update iteration and the adversarial examples are updated by the average gradient of the $K$ sampled examples rather than directly updated by the gradient of the start point of each iteration. By doing this, the angle between the actual and steepest update directions can be reduced, and the update oscillation can be mitigated, thereby enhancing the transferability of generated adversarial examples.


\item Third, the network pruning method is proposed to smooth the decision boundary, which can further mitigate the update oscillation and speed up the convergence during adversarial examples generation in the direction tuning attacks, thereby enhancing the transferability of generated adversarial examples. Particularly, to save computation time, the backward gradient is used to estimate the importance of all neurons through only the one-time feedforward of the network. Then, the low-importance neurons will be pruned to smooth the decision boundary of the network.

\item Finally, the experiments on ImageNet demonstrate that the adversarial examples generated by our methods can significantly improve the average attack success rate (ASR) from 87.9\% to 94.5\% on five victim models without defenses, and from 69.1\% to 76.2\% on eight advanced defense methods in comparison with the latest gradient-based attacks. Besides, when our direction tuning attack is integrated with various input transformations, feature importance-aware attacks, model-specific attacks and model diversification methods, the average ASR of generated adversarial examples can be improved from the best 95.4\% to 97.8\%.
\end{itemize}

The remaining content of this paper is organized as follows: Section~\ref{sec:related-work} introduces the related work and Section~\ref{sec:preliminaries} introduces the notations and basic idea of the previous studies. Section~\ref{sec:methodology} shows the introduction and analysis of our direction tuning attack and network pruning method. Section~\ref{sec:experiments} shows the experimental setup, the experiment results and their analysis. We conclude this paper in Section~\ref{sec:conclusion}.

\section{Related Work}
\label{sec:related-work}

\subsection{Transferable Adversarial Attacks}
\label{sec:transferable-attacks}
Transferable adversarial attacks can be divided into five categories, including gradient-based attacks, feature importance-aware attacks, input transformation methods, model-specific attacks, and model diversification methods. Specifically, the gradient-based attacks contain MI-FGSM~\cite{MI-FGSM}, NI-FGSM~\cite{SI-NI-FGSM}, VMI-FGSM~\cite{VMI-FGSM} and VNI-FGSM~\cite{VMI-FGSM}, which enhances the transferability of adversarial examples by stabilizing the gradient direction from the perspective of the optimization methods. 
Instead of destroying the features of the output layer and avoiding the overfit of the deep layer, Feature Importance-aware Attacks (FIA) such as FIA~\cite{FIA} and Neuron Attribution-based Attacks~\cite{NAA} destroy the intermediate features of DNN to ameliorate the transferability of adversarial examples. Input transformation methods include Diverse Input Method (DIM)~\cite{DI-FGSM}, Translation-Invariant Method (TIM)~\cite{TI-FGSM}, Scale-Invariant Method (SIM)~\cite{SI-NI-FGSM} and Admix~\cite{Admix}, which avoid the generated adversarial examples falling into the bad local optimum through data augmentation. Model-specific attacks include skip gradient method (SGM)~\cite{SGM} and linear propagation~\cite{LinBP}. SGM~\cite{SGM} used more gradient from the skip connections rather than the residual modules in ResNet-like models~\cite{ResNet} to craft adversarial examples with high transferability. Linear propagation~\cite{LinBP} omitted a part of the non-linear layer to perform backpropagation linearly. Model diversification methods include More Bayesian (MB)~\cite{MoreBayesian} and large geometric vicinity~\cite{LGV}, which finetune the pre-trained substitute model to collect diverse models. 

In brief, the five categories of transfer-based attacks enhance the transferability of the generated adversarial examples from different perspectives, i.e., optimization methods, effective intermediate feature, data augmentation, model architecture and model ensemble, respectively, which can be used in combination to further improve the transferability.

In addition, reverse adversarial perturbation (RAP)~\cite{RAP} firstly regarded the transfer-based attack as a min-max problem to find a flat minimum point. Self-Ensemble~\cite{Self-ensemble} and Token Refinement~\cite{Self-ensemble} proposed to enhance the adversarial transferability of vision transformers. Dynamic cues~\cite{Dynamic-Cues} are used to enhance the transferability from the vision transformer-based image model to the video model.

Unlike the existing efforts, the proposed direction tuning added an inner loop sampling module to sample $K$ examples with small steps, and the average gradient of these $K$ sampled examples is used to update the generated adversarial examples in each update iteration, which can effectively increase the degree of gradient alignment~\cite{gradient-alignment}, thereby enhancing the transferability of generated adversarial examples. Moreover, our method is a gradient-based attack, which can be combined with other types of transfer-based attacks such as various input transformations, feature importance-aware attacks (e.g., FIA~\cite{FIA}), model-specific attacks (e.g., SGM~\cite{SGM}) and model diversification methods (e.g., MB~\cite{MoreBayesian}) to further improve the transferability of the generated adversarial examples.


\subsection{Adversarial Defenses}
\label{sec:adversarial-defenses}
Adversarial training~\cite{Madry-AT,Fast-AT,Free-AT,Trades,Ensemble-Adversarial-Training,Confidence-Threshold-Reduction} is the most effective method of defense, which improves the robustness of DNN by optimizing the network parameters with the generated adversarial examples but decreases the natural accuracy a lot. Input processing, including input transformations~\cite{RP,Bit-Red,JPEG,RS} and input purifications~\cite{FD,NRP}, is a class of defense method without changing the parameters of DNN, which make the imperceptible perturbation invalid. To validate the strong transferability of our method against adversarial defenses, the advanced defenses are selected for evaluation in experiments.

\begin{table}
	\centering \caption{Notation}\label{Notation}
	\begin{tabular}{lp{0.7\textwidth}} \hline
		$x,y:$& A clean example and its ground truth label.\\
		$x^{adv}:$& The corresponding adversarial example of $x$.\\
		$f:$& The surrogate model.\\ 
		$\theta:$& The parameters of the surrogate model $f$.\\
		$f(\cdot;\theta):$& The logit vector outputed by $f$.\\
		$L(x,y;\theta):$& The loss function of the surrogate model $f$.\\
		$h:$& The victim model.\\
		$\epsilon:$& The attack strength.\\
		$\alpha:$& The step length.\\
		$g_t:$& The accumulated gradient at the $t$-th iteration.\\
		$x^{adv}_t:$& The generated adversarial example at $t$-th iteration.\\
		$\mu_1:$& The decay factor of $g_t$.\\
		$S^{\alpha}_x:$& The spherical step range, i.e., a ball with center $x$ and radius $\alpha$.\\
		$V^{\alpha}_x:$& The steepest convergence speed at $x$ within $S^{\alpha}_x$.\\
		$\nabla_xL(x,y;\theta):$& The gradient of the loss function w.r.t. $x$.\\
		$\|\nabla_xL(x,y;\theta)\|_2:$& The convergence speed along the gradient $\nabla_xL(x,y;\theta)$.\\
		$K:$& The number of examples sampled in each inner loop of our algorithms.\\
		$\frac{\alpha}{K}:$& The sampling step length in the inner loop of our algorithms.\\
		$x^{adv}_{t,k}:$& The $k$-th sampled example at the $t$-th iteration.\\
		$g_{t,k}:$& The accumulated gradient at the $k$-th iteration in the inner loop and $t$-th iteration in the outer loop of our algorithms.\\
		$\mu_2:$& The decay factor of $g_{t,k}$.\\
		\hline
	\end{tabular}
\end{table}

\section{Preliminaries}
\label{sec:preliminaries}
In this section, the main notations are introduced in Table~\ref{Notation} and the basic idea of latest transfer-based attacks is briefly introduced. 
\subsection{Definition of Adversarial Attack}
\label{sec:notation}
The adversarial attack is defined to generate an  adversarial example (denoted as $x^{adv}$) based on a clean example (denoted as $x$) subjecting to  $\| x-x^{adv}\|_p<\epsilon$, resulting in the adversarial example $x^{adv}$ can mislead the prediction of the classifier $f$, i.e., $f(x;\theta)\ne f(x^{adv};\theta)$. In addition, in our paper,  $f(x;\theta)$  represents the DNN-based classifier with parameters $\theta$ that outputs the logit output of $x$, $L(x,y;\theta)$  represents the loss function of the classifier $f$ (e.g., the cross-entropy loss), $\epsilon$ and $\|\cdot\|_p$  represent the attack strength and $p$-norm distance, respectively. We set $p=\infty$ in this paper to align with previous works.

\subsection{The Gradient-based Attacks}
\label{sec:family-ifgsm}
\textbf{FGSM}~\cite{FGSM} is the first effective gradient-based attack to maximize the loss $L(x^{adv},y;\theta)$ with a step update:
\begin{align}
\label{equation:fgsm}
x^{adv} = x+\epsilon\cdot sign\left(\nabla _xL\left(x,y;\theta\right)\right)
\end{align}
where $sign(\cdot)$ is the sign function and $\nabla_xL(x,y;\theta)$ is the gradient of the loss function w.r.t. $x$.

\textbf{I-FGSM}~\cite{I-FGSM} is an iterative version of FGSM, which further maximizes the loss function $L(x^{adv},y;\theta)$ by using FGSM with a small step in each step update:
\begin{align}
\label{equation:ifgsm}
x^{adv}_{t+1} = Clip^{\epsilon}_x\{x^{adv}_t + \alpha\cdot sign(\nabla_{x^{adv}_{t}}L(x^{adv}_{t},y;\theta))\}
\end{align}
where $x^{adv}_t$ is the generated adversarial example at the $t$-th iteration (note that $x^{adv}_0=x$), $\alpha$ denotes the step length and $Clip^{\epsilon}_x(\cdot)$ function restricts the generated adversarial examples to be within the $\epsilon$-ball of $x$.

\textbf{MI-FGSM}~\cite{MI-FGSM} involves the momentum into I-FGSM~\cite{I-FGSM} to avoid getting stuck into the bad local optimum in adversarial example generation:
\begin{gather}
g_{t+1} = \mu_1 \cdot g_t + \frac{\nabla_{x^{adv}_t}L(x^{adv}_t,y;\theta)}{\|\nabla_{x^{adv}_t}L(x^{adv}_t,y;\theta)\|_1} \label{equation:mifgsm-gradient}, \\
x^{adv}_{t+1} = Clip^{\epsilon}_x\{x^{adv}_t+\alpha\cdot sign(g_{t+1})\} \label{equation:mifgsm-update}
\end{gather}
where $g_t$ is the accumulated gradient at the $t$-th iteration (note that $g_0=0$) and $\mu_1$ is the decay factor of $g_t$.

\textbf{NI-FGSM}~\cite{SI-NI-FGSM} adds the looking ahead operation into MI-FGSM~\cite{MI-FGSM} to further avoid the adversarial example $x^{adv}$ falling into the bad local optimum. Specifically, NI-FGSM firstly calculates the  looking ahead adversarial example $x^{nes}_t$ with  $g_t$ before calculation of the gradient $g_{t+1}$, which can be represented as 
\begin{align}
\label{equation:looking-ahead}
x^{nes}_t = x^{adv}_t+\alpha\cdot \mu_1 \cdot g_t
\end{align}
Note that $\nabla_{x^{adv}_t}L(x^{adv}_t,y;\theta)$ in Eq.~\ref{equation:mifgsm-gradient} is replaced with $\nabla_{x^{adv}_t}L(x^{nes}_t,y;\theta)$.

\textbf{VMI-FGSM}~\cite{VMI-FGSM} and \textbf{VNI-FGSM}~\cite{VMI-FGSM} add variance tuning into MI-FGSM~\cite{MI-FGSM} and NI-FGSM~\cite{SI-NI-FGSM} respectively to further stabilize the gradient direction for reducing update oscillation. Concretely, in VMI-FGSM, the calculation of gradient $g_{t+1}$ is changed by adding the variance $v_t$:
\begin{gather}
g_{t+1} = \mu_1 \cdot g_t + \frac{\nabla_{x^{adv}_t}L(x^{adv}_t,y;\theta)+v_t}{\|\nabla_{x^{adv}_t}L(x^{adv}_t,y;\theta)+v_t\|_1} \label{equation:vmifgsm-gradient}, \\
v_{t+1} = \frac{\sum\nolimits_{i=1}^N{\nabla _{x^{adv}_{t,i}}L(x^{adv}_{t,i},y;\theta)}}{N} - \nabla_{x^{adv}_t}L(x^{adv}_t,y;\theta) \label{equation:vmifgsm-variance}
\end{gather}
where $N$ denotes the number of examples, $v_0=0$, $v_{t}$ denotes the gradient variance at the $t$-th iteration, $x^{adv}_{t,i}=x^{adv}_t+r_i$, $r_i\sim U[-(\beta\cdot\epsilon)^d,(\beta\cdot\epsilon)^d]$, and $U[a^d,b^d]$ stands for the uniform distribution in $d$ dimensions and $\beta$ is a hyperparameter. Note that, for VNI-FGSM, $\nabla_{x^{adv}_t}L(x^{adv}_t,y;\theta)$ in Eq.~\ref{equation:vmifgsm-gradient} is replaced with $\nabla_{x^{adv}_t}L(x^{nes}_t,y;\theta)$.

\subsection{Feature Importance-aware Attack (FIA)}
\label{sec:fia}
Due to the overfitting in the last layers of DNN, FIA~\cite{FIA} improves the transferability of generated adversarial examples by destroying the intermediate features. In fact, FIA can be considered as a gradient-based attack as well,  in which only the loss function is replaced by:
\begin{align}
\label{equation:fia-loss}
L_{FIA}(x,y;\theta)=\sum(\Delta\odot f_l(x;\theta))
\end{align}
where $\Delta$ denotes the importance matrix, $f_l(x;\theta)$ is the feature maps from the $l$-th layer and $\odot$ denotes the element-wise product. It is noteworthy that FIA generates the adversarial examples by minimizing Eq.~\ref{equation:fia-loss}.

\begin{figure}[t]
\begin{center}
\centerline{\includegraphics[width=\columnwidth]{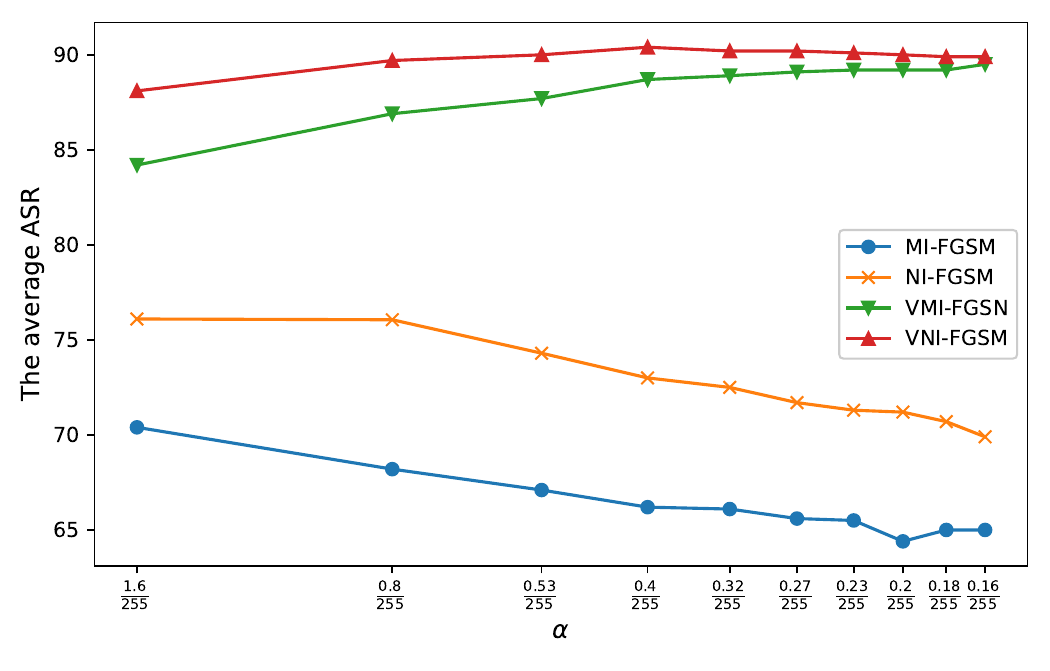}}
\caption{The average attack success rate (ASR) on five models for ImageNet under the black-box setting when reducing the step length $\alpha$ gradually. The adversarial examples are generated by ResNet50~\cite{ResNet}. Note that, under the same attack strength $\epsilon$ constraint, to effective attack, the smaller the attack step length $\alpha$, the larger the number of the step updates $T$.}
\label{fig:different_eps_steps}
\end{center}
\end{figure}

\section{Our Approach}
\label{sec:methodology}
In this section, the motivation of our paper is mentioned first. Then, according to the motivation, our methods including direction tuning attacks and network pruning method are described in details.

\subsection{Motivation}
\label{sec:motivation}
In the white-box attacks, the adversarial examples usually can converge well by setting smaller step lengths $\alpha$, e.g., CW~\cite{CW} and PGD-1000~\cite{Madry-AT} attacks. However, in the transfer-based attacks~\cite{MI-FGSM,SI-NI-FGSM}, the adversarial examples generated with the small step length $\alpha$ will get stuck in the bad local optimum due to the update oscillation occurred by the small step length $\alpha$. For example, in Fig.~\ref{fig:different_eps_steps},  as the attack step length $\alpha$ reduces, the attack success rate of MI-FGSM~\cite{MI-FGSM} and NI-FGSM~\cite{SI-NI-FGSM} decreases gradually, which means the update oscillation leads to poor transferability of the generated adversarial examples. Although variance tuning~\cite{VMI-FGSM} can stabilize the update direction to avoid generated adversarial examples falling into the bad local optimum, the attack success rate of VMI-FGSM~\cite{VMI-FGSM} and VNI-FGSM~\cite{VMI-FGSM} only slightly increases first and then remain unchanged or slightly decreased as the attack step length $\alpha$ reduces because of the update oscillation occurred as well by the small step length $\alpha$.

Hence, the current transfer-based attacks~\cite{MI-FGSM,SI-NI-FGSM,VMI-FGSM} need a large step length of the update to achieve great convergence of adversarial example generation. For example, in the transfer-based attacks, e.g., MI-FGSM~\cite{MI-FGSM} and NI-FGSM~\cite{SI-NI-FGSM}, as shown in Fig.~\ref{fig:different_eps_steps}, the step length $\alpha$ is usually set as $\epsilon/10$ to achieve high attack success rate, which is larger than $\epsilon/1000$ in PGD-1000~\cite{Madry-AT}. 

However, a large update step length $\alpha$ will lead to a big convergence speed deviation along the actual and steepest update directions, which is proved by Theorem~\ref{theorem1} in our paper. According to Theorem~\ref{theorem1}, when the step length $\alpha$ is increased from $\alpha_1$ to $\alpha_2$, the convergence speed deviation along the actual gradient direction and the steepest direction of $x^{adv}_t$ at $t$-th iteration will be increased as well, i.e., $V^{\alpha_2}_{x^{adv}_t}-\|\nabla _{x^{adv}_t}L(x^{adv}_t,y;\theta)\|_2\geq V^{\alpha_1}_{x^{adv}_t}-\|\nabla _{x^{adv}_t}L(x^{adv}_t,y;\theta)\|_2$, which leads to the generated adversarial examples cannot converge well, thereby resulting in low transferability of generated adversarial examples to victim models. This leads to contemplating how to utilize the small length of steps to enhance the transferability of generated adversarial examples.

%
\begin{definition}[The steepest convergence speed]
\label{definition1}
Due to the continuity of DNN, the steepest convergence speed of $L(x,y;\theta)$ at $x$ within the spherical step range $S^{\alpha}_x=[x-\alpha, x+\alpha]$ can be represented as $V^{\alpha}_x$:
\begin{align}
\label{equation:steepest-speed}
V^{\alpha}_x = \max \left\{ \frac{\left| L\left( x,y;\theta \right) -L\left( x_1,y;\theta \right) \right|}{\| x-x_1 \|_2} \mid x_1\in S^{\alpha}_x \right\}
\end{align}
where the spherical step range $S^{\alpha}_x$ indicates a ball with center $x$ and radius $\alpha$.
\end{definition}
\begin{definition}[The convergence speed deviation]
\label{definition2}
The convergence speed deviation is defined as the difference between the convergence speed along the steepest and actual gradients at $x$ within the spherical step range $S^{\alpha}_x$, i.e., $V_x^{\alpha}-\|\nabla _xL(x,y;\theta)\|_2$.
\end{definition}
\begin{theorem}
\label{theorem1}
The convergence speed deviation $V^{\alpha}_x-\|\nabla _xL(x,y;\theta)\|_2$ satisfied: when $\alpha_2>\alpha_1$, $V^{\alpha_2}_x-\|\nabla _xL(x,y;\theta)\|_2\geq V^{\alpha_1}_x-\|\nabla _xL(x,y;\theta)\|_2\geq 0$.
\end{theorem}
\begin{proof}
Assuming that $\alpha_1>\alpha_0$ and $\alpha_0\rightarrow 0$, so
\begin{align}
& \alpha_2>\alpha_1>\alpha_0,\,\, V^{\alpha_0}_x\approx \|\nabla_xL(x,y;\theta)\|_2 \nonumber \\
&\Rightarrow S^{\alpha_0}_x\subset S^{\alpha_1}_x\subset S^{\alpha_2}_x 
\Rightarrow V^{\alpha_2}_x\geq V^{\alpha_1}_x\geq V^{\alpha_0}_x \nonumber \\
&\Rightarrow V^{\alpha_2}_x-\|\nabla _xL(x,y;\theta)\|_2\approx V^{\alpha_2}_x-V^{\alpha_0}_x \nonumber \\
&\geq V^{\alpha_1}_x-V^{\alpha_0}_x\approx V^{\alpha_1}_x-\|\nabla _xL(x,y;\theta)\|_2\geq 0 \nonumber \\
&\Rightarrow V^{\alpha_2}_x-\|\nabla _xL(x,y;\theta)\|_2\geq V^{\alpha_1}_x-\|\nabla _xL(x,y;\theta)\|_2\geq 0 \nonumber
\end{align}
\end{proof}

To address the above issue without changing the step length $\alpha$ in the general gradient-based transferable attacks~\cite{I-FGSM,MI-FGSM,SI-NI-FGSM,VMI-FGSM}, in our paper, direction tuning attack embeds the small length of sampling step (i.e., $\frac{\alpha}{K}$) in each update iteration to sample $K$ examples and replace the original gradient with the average gradient of the $K$ sampled examples, which decreases the angle between the update gradient direction and the steepest direction (i.e., increases the gradient alignment~\cite{gradient-alignment}) and mitigate the update oscillation. In addition, the network pruning method is proposed as well to further mitigate the update oscillation, thereby promoting the generated adversarial examples converge well. In the following, the proposed direction tuning attacks and network pruning method are described in detail, respectively.


%
\begin{algorithm}[tb]
	\caption{DTA}
	\label{alg:dtm}
	\begin{algorithmic}[1]
		\Require
		The natural example $x$ with its ground truth label $y$; the surrogate model $f$; the loss function $L$; The magnitude of perturbation $\epsilon$; the number of iterations $T$; the decay factor $\mu_1$;
		\Require
		The number of iterations $K$ and decay factor $\mu_2$ in the inner loop of our \textit{direction tuning}.
		\Ensure
		An adversarial example $x^{adv}$.
		\State $\alpha=\epsilon/T$
		\State $g_0=0; x^{adv}_0=x$
		\For {$t=0\rightarrow T-1$}
		\State $g_{t,0}=g_t;x^{adv}_{t,0}=x^{adv}_t$
		\For {$k=0\rightarrow K-1$}
		\State Get $x^{nes}_{t,k}$ by $x^{nes}_{t,k}=x^{adv}_{t,k}+\alpha\cdot\mu_1\cdot g_{t,k}$ (i.e., Eq.~\ref{equation:looking-ahead-dta})
		\State Get the gradient $g_{t,k+1}$ by Eq.~\ref{equation:dta-inner-gradient} 
		\State Update $x^{adv}_{t,k+1}$ by Eq.~\ref{equation:dta-update}
		\EndFor
		\State Update $g_{t+1}$ by Eq.~\ref{equation:dta-gradient}
		\State Update $x^{adv}_{t+1}$ by Eq.~\ref{equation:mifgsm-update}
		\EndFor
		\State $x^{adv}=x^{adv}_T$\\
		\Return $x^{adv}$
	\end{algorithmic}
\end{algorithm}
\begin{figure}[t]
	\begin{center}
		\centerline{\includegraphics[width=\columnwidth]{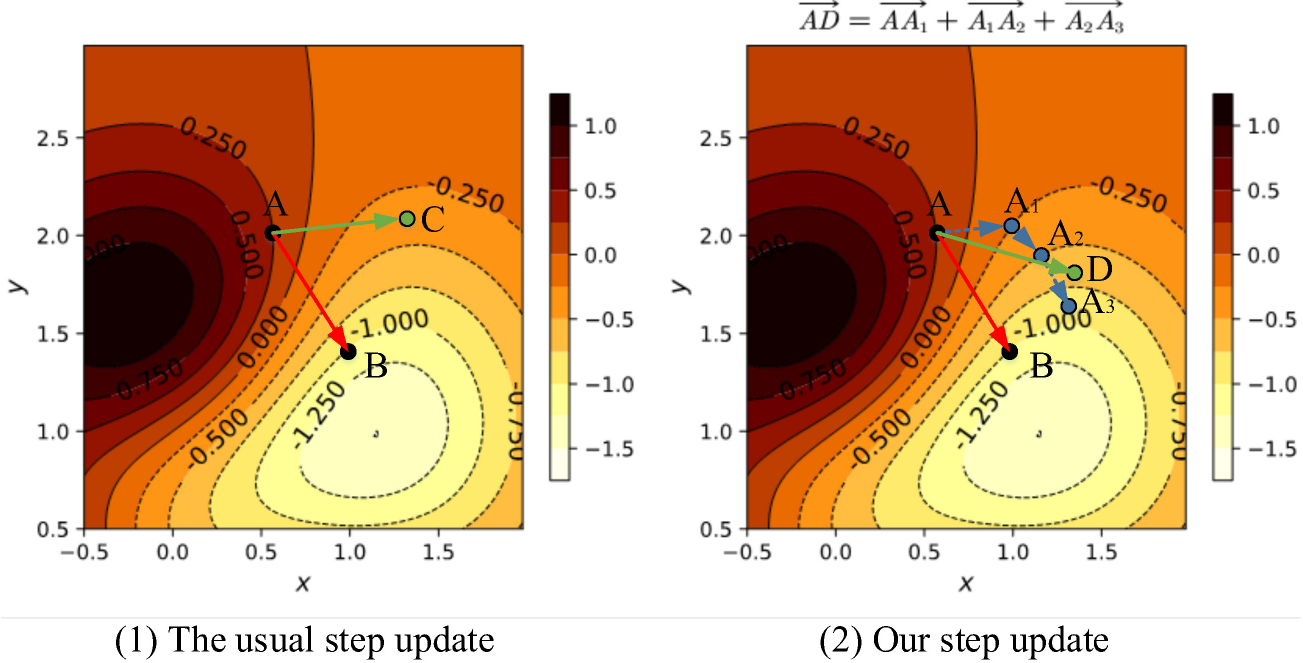}}
		\caption{A comparison between the usual step update and our step update in a function of two variables. In the sub-figures (1) and (2), i.e. the contour line of the function, the red arrow $\protect\overrightarrow{AB}$ indicates the steepest direction, the green arrows $\protect\overrightarrow{AC}$ and $\protect\overrightarrow{AD}$ indicate the actual directions of the usual step update and our step update, respectively. Specifically, the arrow $\protect\overrightarrow{AD}$ is achieved by combining the gradient vector $\protect\overrightarrow{AA_1}$ of point $A$ with two gradient vectors (i.e., $\protect\overrightarrow{A_1A_2}$ and $\protect\overrightarrow{A_2A_3}$) of the sampling points (i.e., $A_1$ and $A_2$, which are achieved with the small sampling step length). Note that the direction of vector $\protect\overrightarrow{AC}$ is the same as that of vector $\protect\overrightarrow{AA_1}$.}
		\label{fig:a_example_of_direction_tuning}
	\end{center}
\end{figure}

\subsection{Direction Tuning Attack}
\label{sec:direction-tuning}
\subsubsection{Basic Idea of Direction Tuning Attack}
\label{sec:direction-tuning-intro}
According to the analysis in Section~\ref{sec:motivation}, in the current transfer-based attacks, the large step length $\alpha$ leads to the big deviation between the convergence speeds of $L(x,y;\theta)$ along the actual update direction and the steepest update direction. To reduce the deviation, \textit{direction tuning attack} (DTA) is proposed to add an inner loop in each step update, i.e., embedding multiple small sampling step lengths into each large step length of the update. Particularly, the small step length is used to sample multiple examples in each update. Then the average gradient of the sampled examples is used for the outer adversarial example update, which can decrease the angle between the actual and steepest update directions and mitigate the update oscillation.

Concretely, in our direction tuning attack, assuming that the number of iterations is $K$ in the inner loop, the gradient at $t$-th step update can be calculated by:
\begin{gather}
g_{t+1} = \mu_1\cdot g_t + \frac{\sum\nolimits_{k=1}^{K}{g_{t,k}}}{K} \label{equation:dta-gradient},
\end{gather}
where $\mu_1\cdot g_t$ is the momentum, the $\frac{\sum\nolimits_{k=1}^{K}{g_{t,k}}}{K}$ is the average gradient of $K$ sampled examples, which are generated with a small step length, i.e. $\frac{\alpha}{K}$. In Eq.~\ref{equation:dta-gradient}, the gradient of each sampled example is calculated by:
\begin{align}
\label{equation:looking-ahead-dta}
x^{nes}_{t,k} = x^{adv}_{t,k}+\alpha\cdot \mu_1 \cdot g_{t,k},
\end{align}
\begin{gather}
\left\{ \begin{array}{l}
g_{t,k+1} = \mu_2\cdot g_{t,k} + \frac{\nabla_{x^{adv}_{t,k}}L(x^{nes}_{t,k},y;\theta)}{\|\nabla_{x^{adv}_{t,k}}L(x^{nes}_{t,k},y;\theta)\|_1} \label{equation:dta-inner-gradient}\\
g_{t,0}=g_t \\
\end{array} \right.,
\end{gather}
where $\mu_2$ is the decay factor of $g_{t,k}$ in the inner loop and $x^{adv}_{t,k}$ denotes the $k$-th sampled example. Eq.~\ref{equation:looking-ahead-dta} indicates the looking ahead operation~\cite{SI-NI-FGSM}, which is applied in the inner loop to select an example having a more accurate gradient direction. In Eq.~\ref{equation:dta-inner-gradient}, each example is sampled by:
\begin{gather}
\left\{ \begin{array}{l}
x^{adv}_{t,k+1} = Clip^{\epsilon}_x\{x_{t,k}^{adv}+\frac{\alpha}{K}\cdot sign(g_{t,k+1})\} \label{equation:dta-update}\\
x^{adv}_{t,0} = x^{adv}_t \\
\end{array} \right.
\end{gather}
where $\frac{\alpha}{K}$ is the small step length in the inner loop. 

Note that, Eq.~\ref{equation:mifgsm-gradient}, \ref{equation:mifgsm-update} and \ref{equation:looking-ahead} in MI-FGSM~\cite{MI-FGSM} and NI-FGSM~\cite{SI-NI-FGSM} are used to update adversarial examples in each iteration with the step length as $\alpha$. However, Eq.~\ref{equation:looking-ahead-dta}, \ref{equation:dta-inner-gradient} and \ref{equation:dta-update} in the proposed DTA are used to sample $K$ examples with the step length as $\frac{\alpha}{K}$. Then, the adversarial examples in our DTA are updated by the average gradient of the $K$ sampled examples in each iteration, rather than updated directly by the gradient of the start point of each iteration.

Actually, the existing efforts~\cite{gradient-alignment} have been proven that the transferability of the generated adversarial examples can be enhanced through increasing the degree of gradient alignment, in which the degree of gradient alignment is defined as the cosine of the angle between the actual and steepest update directions. Our proposed direction tuning attack (DTA) samples $K$ examples in each update iteration, and the average gradient of these sampled examples can reduce the angle between the actual and steepest update directions, i.e., increasing the degree of gradient alignment. The proposed direction tuning attack (DTA) is shown in Algorithm 1.
 Specifically, firstly, all parameters are initialized (Lines 1$\sim$2). Then, for each iteration, the inner loop samples $K$ examples and calculates their gradients, and the outer loop updates the adversarial example with the average gradient of these sampled examples (Lines 5$\sim$10). Finally, the adversarial example is generated at the $T$-th iteration. Particularly, in the inner loop sampling (Lines 5$\sim$10), $K$ examples are sampled with a small step length $\frac{\alpha}{K}$. 

Fig.~\ref{fig:a_example_of_direction_tuning} shows an intuitive instance of an update of our direction tuning attack, in the usual step update shown in Fig.~\ref{fig:a_example_of_direction_tuning}-(1), the actual update direction $\protect\overrightarrow{AC}$ seriously deviates from the steepest update direction $\protect\overrightarrow{AB}$, which leads to poor and slow convergence. In our step update shown in Fig.~\ref{fig:a_example_of_direction_tuning}-(2), the gradient vectors of two sampled examples, i.e. $\protect\overrightarrow{A_1A_2}$ and $\protect\overrightarrow{A_2A_3}$, are used to produce a better update direction (i.e., $\overrightarrow{AD}$) to decrease the angle between $\protect\overrightarrow{AC}$ and $\protect\overrightarrow{AB}$, i.e. $\left<\protect\overrightarrow{AD}, \protect\overrightarrow{AB}\right><\left<\protect\overrightarrow{AC}, \protect\overrightarrow{AB}\right>$ where $\left<\cdot, \cdot\right>$ denotes the angle between two directions. The average gradient $\overrightarrow{AD}$ can reduce the update oscillation of $\protect\overrightarrow{AA_1}$, $\protect\overrightarrow{A_1A_2}$ and $\protect\overrightarrow{A_2A_3}$, which are the gradient of the sampled examples.

To achieve great comprehensive performance, our direction tuning attack can be updated by integrating with the variance tuning~\cite{VMI-FGSM}, namely variance tuning-based direction tuning attack (VDTA), in which the gradient calculation formula in VDTA is changed as:
\begin{gather}
\left\{ \begin{array}{l}
g_{t,k} = \mu_2\cdot g_{t,k-1} + \frac{\nabla_{x^{adv}_{t,k-1}}L(x^{nes}_{t,k-1},y;\theta)+v_{t,k-1}}{\|\nabla_{x^{adv}_{t,k-1}}L(x^{nes}_{t,k-1},y;\theta)+v_{t,k-1}\|_1} \label{equation:vdta-inner-gradient}\\
g_{t,0} = g_t \\
\end{array} \right.,
\end{gather}
where $v_{t,k-1}$ denotes the gradient variances of the $(k-1)$-th sampled example $x^{adv}_{t,k-1}$ in the inner loop of the $t$-th step update, which is calculated by:
\begin{gather}
\left\{ \begin{array}{l}
v_{t,k-1} = \frac{\sum\nolimits_{i=1}^N{\nabla _{x^{adv}_{t,k-2,i}}L(x^{adv}_{t,k-2,i},y;\theta)}}{N} - \nabla_{x^{adv}_{t,k-2}}L(x^{adv}_{t,k-2},y;\theta) \label{equation:vdta-variance-tuning}\\
v_{t,0}=v_t \\
x^{adv}_{t,k-2,i}=x^{adv}_{t,k-2}+r_i\\
\end{array} \right.,
\end{gather}
where $v_t$ is calculated by Eq.~\ref{equation:vmifgsm-variance}, and $r_i$ is a mean distributed white noise, i.e. $r_i\sim U[-(\beta\cdot\epsilon)^d,(\beta\cdot\epsilon)^d]$. 

Note that, VMI-FGSM~\cite{VMI-FGSM} updates adversarial examples in each iteration with the step length as $\alpha$ using Eq.~\ref{equation:mifgsm-update}, \ref{equation:vmifgsm-gradient} and \ref{equation:vmifgsm-variance}, but the proposed VDTA uses Eq.~\ref{equation:dta-update}, \ref{equation:vdta-inner-gradient} and \ref{equation:vdta-variance-tuning} to sample $K$ examples with the step length as $\frac{\alpha}{K}$. The pseudocode of VDTA is shown in \ref{appendix:pseudocode-vdta}.

\begin{figure}[t]
\begin{center}
\centerline{\includegraphics[width=\columnwidth]{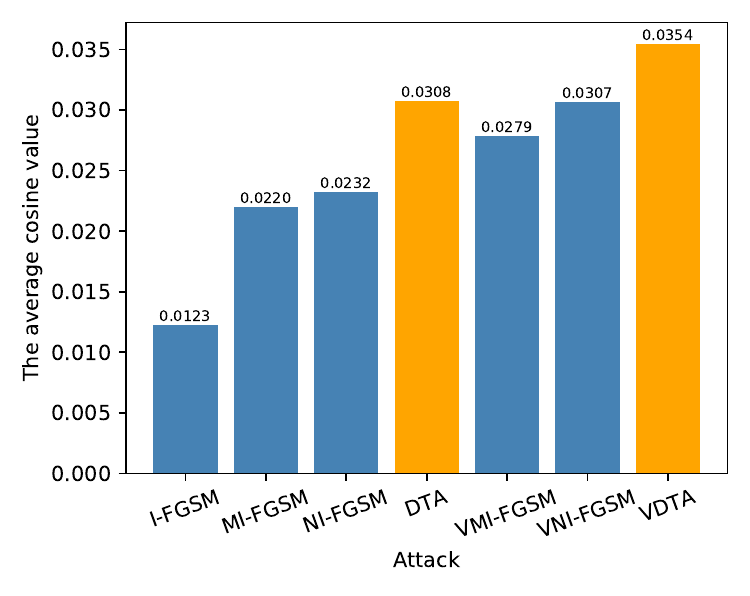}}
\caption{The average cosine value of the angle between the actual and steepest update directions among different attacks on ImageNet~\cite{ImageNet}. The surrogate model is ResNet50~\cite{ResNet} and the victim model is VGG16~\cite{VGG}.}
\label{fig:cosine_curve}
\end{center}
\end{figure}

\subsubsection{Analysis of Direction Tuning Attack}
\label{sec:dta-analysis}
To analyze the effect of our direction tuning attack on reducing the angle between the actual and steepest update directions, the cosine value of the angle between the actual and steepest gradient directions among different attacks on ImageNet~\cite{ImageNet} is compared in Fig.~\ref{fig:cosine_curve}. In a transfer-based attack, the actual gradient is generated by the attack on the surrogate model, and the steepest gradient is approximately the descent gradient on the victim model to the input. This is reasonable because gradient alignment~\cite{gradient-alignment} indicated that the more aligned the gradient generated by both the surrogate and victim models, the higher the transferable attack success rate. As shown in Fig.~\ref{fig:cosine_curve}, the average cosine value of the angle between the actual and steepest gradient directions (i.e., $\nu$) is calculated by:
\begin{gather}
\label{equation:average-cosine-value}
\nu = \sum_{i=1}^M{\sum_{t=1}^T{\cos \left( g_{i,t}^{f},g_{i,t}^{h} \right)}}
\end{gather}
where $M$ is the size of the test set, $T$ denotes the number of iterations, $g_{i,t}^f$ denotes the actual gradient of the $i$-th example at $t$-th step update generated by an attack on the surrogate model $f$, and $g_{i,t}^h$ denotes the corresponding steepest gradient generated by gradient descent method on the victim model $h$.

Fig.~\ref{fig:cosine_curve} shows that our DTA can significantly reduce the angle between the actual and steepest update directions when compared with I-FGSM and its variants. When variance tuning~\cite{VMI-FGSM} is integrated into our direction tuning attack (i.e., VDTA achieved), the cosine value of VDTA is also significantly higher than that of other attacks. In addition, the higher the cosine value of the angle between the actual and steepest update directions, the greater the transferable attack success rate, which will be verified in Section~\ref{sec:experiments}.

In addition to the accurate update direction property as described above, our attack also has the oscillation reduction property. That is our attack can also eliminate the update oscillation when the step update length is small, thereby enhancing the ability of the generated adversarial example to get out of the bad local optimum. The accurate update direction and oscillation reduction properties of our direction tuning attack are analyzed mathematically as follows.


The current transfer-based attacks update $x^{adv}_t$ by adding $\alpha\cdot sign(g_{t,1})$, i.e., 
\begin{gather}
\label{equation:usual-step-update}
x^{adv}_{t+1} = x^{adv}_t + \alpha\cdot sign(g_{t,1}),
\end{gather}
where $g_{t,1}$ is the gradient calculated by the attack method at the point $x^{adv}_t$ on the surrogate model. In fact, reducing the step update length can make the generated adversarial example convergence better that is why the attack success rate of I-FGSM~\cite{I-FGSM} is higher than FGSM~\cite{FGSM}. Hence, when the update step length $\alpha$ is divided into multiple small step lengths, the adversarial example $x^{adv}_{t+1}$ is updated by
%
\begin{gather}
\label{equation:smaller-step-update}
x^{adv}_{t+1} = x^{adv}_t + \frac{\alpha}{K}\cdot [sign(g_{t,1})+sign(g_{t,2})+\cdots+sign(g_{t,K})],
\end{gather}
where $\frac{\alpha}{K}$ denotes the small step length, and $g_{t,i}$ is the gradient calculated by the attack method at the point $[x^{adv}_t+\frac{\alpha}{K}\cdot sign(g_{t,1})+\frac{\alpha}{K}\cdot sign(g_{t,2})+\cdots+\frac{\alpha}{K}\cdot sign(g_{t,i-1})]$ on the surrogate model. However, $\frac{\alpha}{K}$ may be too small to promote the generated adversarial example jump out of the bad local optimum due to the update oscillation. To solve the issue, in our attack, the Eq.~\ref{equation:smaller-step-update} is changed by:
\begin{gather}
\label{equation:direction-tuning-step-update}
x^{adv}_{t+1} = x^{adv}_t + \alpha\cdot sign(\frac{g_{t,1}+g_{t,2}+\cdots+g_{t,K}}{K}).
\end{gather}
In most cases, $\alpha\cdot sign(\frac{g_{t,1}+g_{t,2}+\cdots+g_{t,K}}{K})$ in Eq.~\ref{equation:direction-tuning-step-update} has a smaller angle with $\frac{\alpha}{K}\cdot (sign(g_{t,1})+sign(g_{t,2})+\cdots+sign(g_{t,K}))$ in Eq.~\ref{equation:smaller-step-update} in comparison with $\alpha\cdot sign(g_{t,1})$ in Eq.~\ref{equation:usual-step-update}, i.e.,
\begin{gather}
\cos(\alpha\cdot sign(\frac{g_{t,1}+g_{t,2}+\cdots+g_{t,K}}{K})\nonumber\\
, \frac{\alpha}{K}\cdot (sign(g_{t,1})+sign(g_{t,2})+\cdots+sign(g_{t,K})))\geq \nonumber\\
\cos(\alpha\cdot sign(g_{t,1}), \frac{\alpha}{K}\cdot (sign(g_{t,1})+sign(g_{t,2})+\cdots+sign(g_{t,K})))\label{equation:accurate-update-direction-property}.
\end{gather}
%
Moreover, in most cases, the updates in Eqs.~\ref{equation:direction-tuning-step-update} and \ref{equation:usual-step-update} have the same moving distance along their respective update directions, but the moving distance of update in Eq.~\ref{equation:smaller-step-update} becomes small because of the existence of update oscillation (i.e., the directions of the gradients $g_{t,1}, g_{t,2}, \cdots, g_{t,K}$ have angles greater than 0 to each other), i.e.,
\begin{gather}
\left\| \alpha\cdot sign(\frac{g_{t,1}+g_{t,2}+\cdots+g_{t,K}}{K})\right\|_2=\left\| \alpha\cdot sign(g_{t,1})\right\|_2\geq \nonumber\\
\left\| \frac{\alpha}{K}\cdot (sign(g_{t,1})+sign(g_{t,2})+\cdots+sign(g_{t,K}))\right\|_2 \label{equation:oscillation-reduction-property}.
\end{gather}
Before the updation in Eq.~\ref{equation:direction-tuning-step-update}, $sign(\frac{g_{t,1}+g_{t,2}+\cdots+g_{t,K}}{K})$ eliminates the oscillating effects of these gradients, i.e. the components of $g_{t,1}, g_{t,2}, \cdots, g_{t,K}$ perpendicular to $(g_{t,1}+g_{t,2}+\cdots+g_{t,K})$ cancel each other out, which verifies the oscillation reduction property of Eq.\ref{equation:direction-tuning-step-update}.

Therefore, according to Eq.~\ref{equation:accurate-update-direction-property} and \ref{equation:oscillation-reduction-property}, the update equation (i.e., Eq.~\ref{equation:direction-tuning-step-update}) in our attack simultaneously have the accurate update direction, which is similar to the update equation with the small update step length shown in Eq.~\ref{equation:smaller-step-update} and enough moving distance by mitigating oscillation. 


\subsection{Network Pruning Method}
\label{sec:network-pruning}
\subsubsection{Basic idea of Network Pruning Method}
\label{sec:network-pruning-intro}
Network pruning~\cite{Network-pruning} can compress the neural network while maintaining the network performance. Recently, network pruning is used to improve the robustness of DNN~\cite{SAP}. Moreover, our study also found that network pruning can be utilized to enhance the transferability of the adversarial examples to attack DNN. 

Hence, the network pruning method (NP) is proposed in this paper to enhance the transferability of adversarial examples. To be specific, in our network pruning method, the redundant or unimportant neurons in DNN will be pruned while keeping the correct classification for the input, thereby smoothing the decision boundary. We can also understand it as focusing on destroying the important features extracted by the remaining neurons. To evaluate the importance of the neurons, the evaluation metrics $\boldsymbol{I}_{i,j}$ is introduced as the absolute variation of the loss value after the neuron is pruned, which can be represented as
\begin{align}
\label{equation:importance-metric}
\boldsymbol{I}_{i,j} = \left|L(x,y;\theta\backslash f_{i,j})-L(x,y;\theta)\right|
\end{align}
where $f_{i,j}$ denotes the $j$-th neuron in the $i$-th layer, $\theta\backslash f_{i,j}$ represents deleting the neuron $f_{i,j}$ from the classifier $f(x;\theta)$ and $\boldsymbol{I}$ is the importance matrix.

To calculate each element in $\boldsymbol{I}$ with Eq.~\ref{equation:importance-metric}, an extra feedforward of $f(x;\theta\backslash f_{i,j})$ is needed, which needs a lot of extra computation overhead. To save computation time, the backward gradient is used to estimate the importance of all neurons through only the one-time feedforward of $f(x;\theta)$:
\begin{align}
\label{equation:estimated-importance-metric}
\boldsymbol{I}_{i,j} = \left|\frac{\partial L(x,y;\theta)}{\partial f_{i,j}(x;\theta)}\right|
\end{align}

After calculating the matrix $\boldsymbol{I}$, the neurons with the least importance are pruned according to the pruning rate $\gamma$ in each step update during adversarial example generation. Then, the update direction is calculated by the remaining DNN. Note that, only the last feature layer in DNN is pruned, e.g., the last layer of the fourth block for ResNet50~\cite{ResNet}. 

\subsubsection{Analysis of Network Pruning Method}
\label{sec:network-pruning-analysis}
\textit{Why the last feature layer is selected to prune?} Due to the classification on the last feature layer having the highest accuracy, the features of different categories in the last feature layer have little correlation. Hence, pruning redundant features in the last feature layer does not influent correct classification, so our attacks can focus on destroying the more important features.

In addition, our network pruning method can decrease the expressive power of DNN by reducing the size of parameters, thereby simplifying the classification boundary of DNN. Therefore, the update oscillation during generating adversarial examples can be further decreased for good convergence, thereby enhancing the transferability of generated adversarial examples.


\section{Experiments}
\label{sec:experiments}
To validate the effectiveness of our methods, the experiments are conducted on ImageNet~\cite{ImageNet} in this section. The experimental setup is firstly specified in Section~\ref{sec:experimental-setup}. Section~\ref{sec:comparison-direction-tuning} verifies the transferability of our direction tuning attacks on victim models with and without defenses, respectively. Section~\ref{sec:comparison-network-pruning} demonstrates that the network pruning method (NP) can further improve the transferability of adversarial attacks. Section~\ref{sec:combination-fia-and-input-transformation} represents that our methods can be well combined with other types of transfer-based attacks. Section~\ref{sec:sensitivity-analysis} makes a series of sensitivity analyses on parameter $K$ and $\mu_2$ of our direction tuning attacks, and parameter $\gamma$ of our network pruning method. To highlight the contribution of our methods, we compare them with the other comparative attacks under the same time complexity in Section~\ref{sec:the-same-of-time-complexity}.


\subsection{Experimental Setup}
\label{sec:experimental-setup}
\textbf{Datasets.} 1,000 clean images are randomly picked from ImageNet~\cite{ImageNet} validation set and correctly classified by all the testing models. Note that, due to ImageNet~\cite{ImageNet} is a large real dataset that includes the small datasets in real world, the validity of our method on ImageNet is sufficient to prove that our method can also achieve great performance on the small real-world datasets such as CIFAR~\cite{CIFAR}. Hence, similar to the advanced works~\cite{MI-FGSM,SI-NI-FGSM,VMI-FGSM}, ImageNet is used to demonstrate the effectiveness of their methods sufficiently.

\textbf{Models.} Six naturally trained models are used, which are VGG16/19~\cite{VGG}, ResNet50/152 (Res50/152)~\cite{ResNet}, MobileNet-v2 (Mob-v2)~\cite{MobileNet} and Inception-v3 (Inc-v3)~\cite{Inception-v3}. These are the pre-trained models in \cite{pytorch-image-models,torchvision-models} on the ImageNet dataset. Particularly, VGG16 and ResNet50 are selected as surrogate models. Additionally, eight advanced defense models are used to further evaluate the strong transferability of our attacks, including Resized and Padding (RP)~\cite{RP}, Bit Reduction (Bit-Red)~\cite{Bit-Red}, JPEG compression (JPEG)~\cite{JPEG}, Feature Distillation (FD)~\cite{FD}, Randomized Smoothing (RS)~\cite{RS}, Neural Representation Purifier (NRP)~\cite{NRP}, adversarial Inception-v3 (Inc-v3$^{adv}$) and ensemble adversarial Inception-ResNet-v2 (IncRes-v2$^{adv}_{ens}$)~\cite{Ensemble-Adversarial-Training}, wherein the first six defenses adopt MobileNet-v2 as the victim model.

\textbf{Baselines.} In the evaluations, five gradient-based attacks are considered as our baselines, including I-FGSM~\cite{I-FGSM}, MI-FGSM~\cite{MI-FGSM}, NI-FGSM~\cite{SI-NI-FGSM}, VM(N)I-FGSM~\cite{VMI-FGSM}. In addition, our method is also integrated with feature importance-aware attacks (i.e., FIA~\cite{FIA}), various input transformation methods (i.e., DIM~\cite{DI-FGSM}, SIM~\cite{SI-NI-FGSM} and TIM~\cite{TI-FGSM}), SGM~\cite{SGM} and More Bayesian (MB)~\cite{MoreBayesian} to evaluate the compatibility property.

\textbf{Metrics.} The average ASR on several victim models under the black-box setting (i.e. without any information of the victim models) is considered as the metric to indicate the  transferability of the attacks.

\textbf{Hyper-parameters.} The maximum perturbation, number of iterations, and step length are set as $\epsilon=16/255, T=10, \alpha= 1.6/255$ on ImageNet respectively, which is similar to MI-FGSM \cite{MI-FGSM}. For five gradient-based attacks, the decay factor is set to 1.0, i.e., $\mu_1=1.0$. For DIM, the transformation probability is set to 0.5. For TIM, the Gaussian kernel is selected and the size of the kernel is set to $15\times 15$. For SIM, the number of scale copies is set to 5. For variance tuning-based attacks, the factor $\beta$ for the upper bound of the neighborhood is set to 1.5, i.e., $\beta=1.5$ and the number of examples $N$ for variance tuning is set to 20, i.e., $N=20$. For FIA, the drop probability is set to 0.1, the ensemble number in aggregate gradient is set to 30, and the last layer of the third block for ResNet50 is chosen. For SGM, the decay factor is set to 0.2. For our direction tuning attacks, the parameters are set as $K=10, \mu_2=0.0$ on DTA and $K=10, \mu_2=0.8$ on VDTA, respectively. For the network pruning method, the pruning rate is set to 0.9, i.e., $\gamma=0.9$.

\begin{table}[t]
\caption{The ASR (\%) comparison on six models on ImageNet~\cite{ImageNet} dataset. The adversarial examples are crafted by the surrogate model, VGG16~\cite{VGG} or ResNet50~\cite{ResNet}. $*$ denotes the ASR under the white-box setting. \textbf{Average} means the average value except $*$.}
\label{tab:direction-tuning}
\scriptsize
\begin{center}
\resizebox{1.0\textwidth}{!}{
\begin{tabular}{ccccccccc}
\hline
Model                         & Attack  & VGG16          & VGG19         & Res50         & Res152        & Inc-v3         & Mob-v2        & \textbf{Average} \\ \hline
\multirow{7}{*}{VGG16}        & IFGSM   & 99.3*          & 96.6          & 39.3          & 27.5          & 23.2           & 55.6          & 48.4             \\
                              & MIFGSM  & 99.3*          & 98.0          & 65.9          & 49.7          & 49.5           & 77.9          & 68.2             \\
                              & NIFGSM  & \textbf{99.6*} & 98.6          & 69.4          & 53.0          & 50.3           & 82.9          & 70.8             \\
                              & DTA     & \textbf{99.6*} & \textbf{99.3} & \textbf{87.0} & \textbf{71.4} & \textbf{65.9}  & \textbf{90.7} & \textbf{82.9}    \\ \cline{2-9} 
                              & VMIFGSM & 99.8*          & 99.2          & 81.3          & 68.1          & 63.6           & 87.9          & 80.0             \\
                              & VNIFGSM & \textbf{99.8*} & 99.3          & 83.3          & 70.1          & 65.6           & 91.3          & 81.9             \\
                              & VDTA    & \textbf{99.8*} & \textbf{99.5} & \textbf{91.8} & \textbf{81.2} & \textbf{77.1}  & \textbf{94.9} & \textbf{88.9}    \\ \hline
\multirow{7}{*}{ResNet50}     & IFGSM   & 42.1           & 40.3          & \textbf{100*} & 61.6          & 27.0           & 46.6          & 43.5             \\
                              & MIFGSM  & 74.4           & 69.9          & \textbf{100*} & 81.2          & 54.0           & 73.1          & 70.5             \\
                              & NIFGSM  & 80.0           & 77.0          & \textbf{100*} & 87.0          & 55.8           & 78.7          & 75.7             \\
                              & DTA     & \textbf{90.0}  & \textbf{88.8} & \textbf{100*} & \textbf{97.2} & \textbf{72.1}  & \textbf{90.6} & \textbf{87.7}    \\ \cline{2-9} 
                              & VMIFGSM & 86.8           & 85.2          & \textbf{100*} & 92.1          & 71.5           & 87.0          & 84.5             \\
                              & VNIFGSM & 90.0           & 88.9          & \textbf{100*} & 95.6          & 74.7           & 90.1          & 87.9             \\
                              & VDTA    & \textbf{94.6}  & \textbf{93.5} & \textbf{100*} & \textbf{98.5} & \textbf{84.3}  & \textbf{95.7} & \textbf{93.3}    \\ \hline
\multirow{7}{*}{Inception-v3} & IFGSM   & 15.1           & 13.7          & 13.8          & 11.1          & 98.7*          & 19.3          & 14.6             \\
                              & MIFGSM  & 46.6           & 42.5          & 33.9          & 28.9          & 98.6*          & 47.4          & 39.9             \\
                              & NIFGSM  & 53.1           & 49.9          & 43.1          & 34.6          & 99.1*          & 56.4          & 47.4             \\
                              & DTA     & \textbf{60.2}  & \textbf{56.7} & \textbf{53.4} & \textbf{49.9} & \textbf{99.3*} & \textbf{62.9} & \textbf{56.6}    \\ \cline{2-9} 
                              & VMIFGSM & 59.3           & 56.3          & 52.7          & 48.7          & 99.2*          & 61.8          & 55.8             \\
                              & VNIFGSM & 65.1           & 62.1          & 60.1          & 55.6          & \textbf{99.7*} & 67.8          & 62.1             \\
                              & VDTA    & \textbf{70.8}  & \textbf{69.1} & \textbf{66.3} & \textbf{59.8} & \textbf{99.7*} & \textbf{73.3} & \textbf{67.9}    \\ \hline
\end{tabular}
}
\end{center}
\end{table}
\begin{table}[t]
\caption{The ASR (\%) comparison on eight advanced defense models on ImageNet~\cite{ImageNet} dataset under the black-box setting. The adversarial examples are generated by the surrogate model, VGG16~\cite{VGG} or ResNet50~\cite{ResNet}. \textbf{Avg.} denotes the average value.}
\label{tab:with-defenses}
\scriptsize
\begin{center}
\resizebox{1.0\textwidth}{!}{
\begin{tabular}{ccccccccccc}
\hline
Model                     & Attack  & RP            & \begin{tabular}[c]{@{}c@{}}Bit\\ -Red\end{tabular} & JPEG          & FD            & NRP           & RS            & \begin{tabular}[c]{@{}c@{}}Inc\\ -$v3^{adv}$\end{tabular}     & \begin{tabular}[c]{@{}c@{}}IncRes\\ -$v2^{adv}_{ens}$\end{tabular} & \textbf{Avg.} \\ \hline
\multirow{7}{*}{VGG16}        & IFGSM   & 46.5          & 41.9          & 37.5          & 38.8          & 30.2          & 34.5          & 19.3          & 9.0               & 32.2             \\
                              & MIFGSM  & 72.3          & 73.4          & 64.6          & 73.3          & 46.4          & 62.5          & 29.0          & 14.2              & 54.46            \\
                              & NIFGSM  & 75.7          & 79.3          & 68.9          & 77.7          & 47.5          & 68.5          & 27.4          & 14.2              & 57.4             \\
                              & DTA     & \textbf{86.6} & \textbf{88.9} & \textbf{82.0} & \textbf{88.6} & \textbf{48.4} & \textbf{75.5} & \textbf{30.6} & \textbf{16.3}     & \textbf{64.6}    \\ \cline{2-11} 
                              & VMIFGSM & 83.4          & 84.9          & 81.4          & 85.9          & 53.1          & 76.0          & 36.5          & 20.0              & 65.2             \\
                              & VNIFGSM & 85.6          & 89.1          & 83.5          & 87.5          & 54.1          & 78.5          & 36.6          & 20.3              & 66.9             \\
                              & VDTA    & \textbf{90.7} & \textbf{93.8} & \textbf{90.6} & \textbf{93.7} & \textbf{56.9} & \textbf{89.5} & \textbf{41.9} & \textbf{23.3}     & \textbf{72.6}    \\ \hline
\multirow{7}{*}{ResNet50}     & IFGSM   & 40.7          & 36.4          & 34.5          & 34.3          & 32.4          & 31.0          & 19.2          & 9.9               & 29.8             \\
                              & MIFGSM  & 68.2          & 69.3          & 63.9          & 70.2          & 49.7          & 59.5          & 29.6          & 16.1              & 53.3             \\
                              & NIFGSM  & 72.4          & 75.0          & 67.2          & 74.6          & 51.6          & 63.5          & 30.8          & 16.3              & 56.4             \\
                              & DTA     & \textbf{84.0} & \textbf{87.3} & \textbf{82.0} & \textbf{87.2} & \textbf{54.6} & \textbf{73.5} & \textbf{33.3} & \textbf{19.9}     & \textbf{65.2}    \\ \cline{2-11} 
                              & VMIFGSM & 81.9          & 84.7          & 80.5          & 85.0          & 59.7          & 77.0          & 40.4          & 28.4              & 67.2             \\
                              & VNIFGSM & 85.9          & 87.8          & 84.0          & 87.3          & 59.5          & 79.5          & 39.7          & 28.8              & 69.1             \\
                              & VDTA    & \textbf{92.4} & \textbf{94.7} & \textbf{91.8} & \textbf{94.5} & \textbf{64.5} & \textbf{91.0} & \textbf{47.2} & \textbf{33.8}     & \textbf{76.2}    \\ \hline
\multirow{7}{*}{Inception-v3} & IFGSM   & 22.0          & 16.5          & 18.9          & 18.0          & 25.9          & 24.5          & 19.4          & 8.9               & 19.3             \\
                              & MIFGSM  & 53.7          & 44.6          & 42.7          & 46.5          & 37.3          & 50.5          & 32.0          & 16.3              & 40.5             \\
                              & NIFGSM  & 59.4          & 53.4          & 48.7          & 54.1          & 37.6          & 53.5          & 32.8          & 16.2              & 44.5             \\
                              & DTA     & \textbf{66.6} & \textbf{61.8} & \textbf{59.4} & \textbf{62.6} & \textbf{39.2} & \textbf{57.5} & \textbf{38.0} & \textbf{20.5}     & \textbf{50.7}    \\ \cline{2-11} 
                              & VMIFGSM & 63.4          & 60.1          & 59.2          & 59.7          & \textbf{40.8} & 55.5          & 41.6          & 26.2              & 50.8             \\
                              & VNIFGSM & 70.5          & 65.6          & 62.5          & 66.2          & 40.0          & 61.0          & 43.5          & 27.4              & 54.6             \\
                              & VDTA    & \textbf{76.0} & \textbf{72.1} & \textbf{68.9} & \textbf{72.1} & 39.8          & \textbf{70.0} & \textbf{47.4} & \textbf{30.5}     & \textbf{59.6}    \\ \hline
\end{tabular}
}
\end{center}
\end{table}

\subsection{Transferability of our Direction Tuning Attacks}
\label{sec:comparison-direction-tuning}
In this subsection, VGG16, ResNet50 and Inception-v3 are considered as the surrogate models to generate adversarial examples for attacking six naturally trained victim models (as shown in the first row of Table~\ref{tab:direction-tuning}). Besides, the generated adversarial examples are also used to attack eight advanced defense models, including six advanced defenses~\cite{RP,Bit-Red,JPEG,RS,FD,NRP} and two adversarially trained models~\cite{Ensemble-Adversarial-Training} (as shown in the first row of Table~\ref{tab:with-defenses}).

We perform seven gradient-based attacks (i.e., I-FGSM~\cite{I-FGSM} and its variants, and the proposed DTA and VDTA) on the surrogate model to generate adversarial examples. The attack success rate (ASR), which is the misclassification rate of  the generated adversarial examples on the victim model, is shown in Tables~\ref{tab:direction-tuning} and \ref{tab:with-defenses}.

Table~\ref{tab:direction-tuning} shows that the ASRs of direction tuning attack (DTA) on all victim models are significantly higher than that of I-FGSM and its corresponding variants under the black-box setting. Additionally, the ASRs of direction tuning attack (DTA) on surrogate models are similar to that of the other comparative attacks under the white-box setting. Specifically, our DTA can improve the average ASR by 12.1\%, 12.0\% and 9.2\% with VGG16~\cite{VGG}, ResNet50~\cite{ResNet} and Inception-v3~\cite{Inception-v3} as surrogate models respectively in comparison with I-FGSM~\cite{I-FGSM} and its variants. When variance tuning~\cite{VMI-FGSM} is integrated into these adversarial attacks, the average ASR of our VDTA is 7.0\%, 5.4\% and 5.8\% higher than that of other attacks with VGG16~\cite{VGG}, ResNet50~\cite{ResNet} and Inception-v3~\cite{Inception-v3} as surrogate models, respectively.


In addition, Table~\ref{tab:with-defenses} shows that our VDTA achieves an average ASR of 76.2\%, surpassing the latest gradient-based attacks by a large margin of 7.1\%. Hence, the results in Table~\ref{tab:with-defenses} demonstrate that our methods can achieve great generalization to attack advanced defense models. Moreover, the increasing threat in the advanced defense models raises a new security issue for the development of more robust deep learning models. 


%
\begin{table}[t]
\caption{The ASR (\%) comparison on six models on ImageNet~\cite{ImageNet} dataset with or without network pruning method. The adversarial examples are generated by the surrogate model, ResNet50~\cite{ResNet}. * denotes the ASR under the white-box setting. \textbf{Average} denotes the average value except *.}
\label{tab:network-pruning}
\scriptsize
\begin{center}
\begin{tabular}{cccccccc}
\hline
Attack  & VGG16         & VGG19         & Res50         & Res152        & Inc-v3        & Mob-v2        & \textbf{Average} \\ \hline
DTA     & 90.0          & 88.8          & \textbf{100*} & 97.2          & \textbf{72.1} & 90.6          & 87.7             \\
DTA+NP  & \textbf{93.1} & \textbf{92.0} & \textbf{100*} & \textbf{97.8} & 70.5          & \textbf{93.1} & \textbf{89.3}    \\ \hline
VDTA    & 94.6          & 93.5          & \textbf{100*} & 98.5          & 84.3          & 95.7          & 93.3             \\
VDTA+NP & \textbf{96.1} & \textbf{94.5} & 99.9*         & \textbf{98.8} & \textbf{87.2} & \textbf{96.1} & \textbf{94.5}    \\ \hline
\end{tabular}
\end{center}
\end{table}

\subsection{The Effectiveness of Network Pruning Method}
\label{sec:comparison-network-pruning}
This subsection explores the impact of the network pruning method (NP) on direction tuning attack (DTA) and variance tuning-based direction tuning attack (VDTA). Table~\ref{tab:network-pruning} shows the ASRs of our DTA and VDTA on six victim models with and without the network pruning method. The results show that with the network pruning method, the average ASR can be further improved by 1.6\% on DTA and 1.2\% on VDTA under the black-box setting while keeping the ASR on the surrogate model, which demonstrates the effectiveness of the network pruning method on enhancing the transferability of adversarial examples and achieving great generalization on the proposed direction tuning attacks.

\begin{table}[t]
\caption{The ASR (\%) comparison on six models on ImageNet~\cite{ImageNet} dataset using the FIA loss instead of the cross-entropy loss. The adversarial examples are generated by the surrogate model, ResNet50~\cite{ResNet}. * denotes the ASR under the white-box setting. \textbf{Average} denotes the average value except *.}
\label{tab:fia}
\scriptsize
\begin{center}
\begin{tabular}{ccccccccc}
\hline
Attack  & Loss & VGG16         & VGG19         & Res50         & Res152        & Inc-v3        & Mob-v2        & \textbf{Average} \\ \hline
MIFGSM  & FIA  & 84.2          & 82.7          & \textbf{100*} & 87.6          & 64.6          & 83.4          & 80.5             \\
NIFGSM  & FIA  & 86.9          & 84.3          & \textbf{100*} & 89.1          & 66.4          & 85.0          & 82.3             \\
DTA     & FIA  & \textbf{93.4} & \textbf{92.8} & \textbf{100*} & \textbf{96.3} & \textbf{81.6} & \textbf{93.5} & \textbf{91.5}    \\ \hline
VMIFGSM & FIA  & 90.8          & 90.4          & \textbf{100*} & 93.5          & 78.3          & 90.9          & 88.8             \\
VNIFGSM & FIA  & 91.9          & 90.7          & \textbf{100*} & 94.8          & 79.4          & 91.4          & 89.6             \\
VDTA    & FIA  & \textbf{94.2} & \textbf{93.6} & \textbf{100*} & \textbf{96.4} & \textbf{83.3} & \textbf{94.2} & \textbf{92.3}    \\ \hline
\end{tabular}
\end{center}
\end{table}
\begin{table}[t]
\caption{The ASR (\%) comparison on six models on ImageNet~\cite{ImageNet}. The gradient-based attacks are enhanced by DIM~\cite{DI-FGSM}, SIM~\cite{SI-NI-FGSM} and TIM~\cite{TI-FGSM}, respectively. The adversarial examples are generated by the surrogate model, ResNet50~\cite{ResNet}. * denotes the ASR under the white-box setting. \textbf{Avg.} means the average value except *.}
\label{tab:input-transformation-methods}
\scriptsize
\begin{center}
\begin{tabular}{ccccccccc}
\hline
\begin{tabular}[c]{@{}c@{}}Input\\ transformations\end{tabular} & Attack  & VGG16         & VGG19         & Res50          & Res152        & Inc-v3        & Mob-v2        & \textbf{Avg.} \\ \hline
\multirow{8}{*}{DIM}                                            & MIFGSM  & 86.6          & 85.5          & \textbf{100*}  & 93.6          & 77.2          & 87.3          & 86.0             \\
                                                                & NIFGSM  & 89.5          & 88.2          & \textbf{100*}  & 94.5          & 76.6          & 90.0          & 87.8             \\
                                                                & DTA     & 97.6          & \textbf{97.7} & \textbf{100*}  & 99.5          & 92.8          & 98.1          & 97.1             \\
                                                                & DTA+NP  & \textbf{98.3} & \textbf{97.7} & \textbf{100*}  & \textbf{99.8} & \textbf{94.2} & \textbf{98.8} & \textbf{97.8}    \\ \cline{2-9} 
                                                                & VMIFGSM & 91.3          & 89.5          & \textbf{99.9*} & 95.4          & 84.0          & 92.4          & 90.5             \\
                                                                & VNIFGSM & 93.4          & 93.3          & \textbf{99.9*} & 96.2          & 85.0          & 94.4          & 92.5             \\
                                                                & VDTA    & 96.3          & 95.2          & \textbf{99.9*} & 97.4          & 90.7          & 96.3          & 95.2             \\
                                                                & VDTA+NP & \textbf{96.7} & \textbf{96.7} & 99.8*          & \textbf{98.5} & \textbf{92.4} & \textbf{97.7} & \textbf{96.4}    \\ \hline
\multirow{8}{*}{SIM}                                            & MIFGSM  & 87.5          & 86.8          & \textbf{100*}  & 93.5          & 74.0          & 89.5          & 86.3             \\
                                                                & NIFGSM  & 90.7          & 89.9          & \textbf{100*}  & 96.4          & 76.9          & 91.2          & 89.0             \\
                                                                & DTA     & 96.2          & 95.3          & \textbf{100*}  & 99.2          & 86.9          & 96.5          & 94.8             \\
                                                                & DTA+NP  & \textbf{97.0} & \textbf{95.6} & \textbf{100*}  & \textbf{99.5} & \textbf{88.4} & \textbf{98.3} & \textbf{95.8}    \\ \cline{2-9} 
                                                                & VMIFGSM & 93.6          & 93.4          & \textbf{100*}  & 98.2          & 86.9          & 95.1          & 93.4             \\
                                                                & VNIFGSM & 95.7          & 95.8          & \textbf{100*}  & 99.2 & 90.1          & 96.3          & 95.4             \\
                                                                & VDTA    & 97.4          & 97.1          & \textbf{100*}  & 99.2          & 92.8          & 98.0          & 96.9             \\
                                                                & VDTA+NP & \textbf{98.2} & \textbf{97.4} & \textbf{100*}  & \textbf{99.5} & \textbf{94.0} & \textbf{98.2} & \textbf{97.5}    \\ \hline
\multirow{8}{*}{TIM}                                            & MIFGSM  & 75.5          & 73.9          & 99.8*          & 74.3          & 58.1          & 72.8          & 70.9             \\
                                                                & NIFGSM  & 78.9          & 76.8          & 99.8*          & 78.5          & 59.8          & 77.9          & 74.4             \\
                                                                & DTA     & \textbf{88.6} & 88.7          & \textbf{100*}  & 93.2          & 74.2          & 89.7          & 86.9             \\
                                                                & DTA+NP  & \textbf{88.6} & \textbf{89.1} & \textbf{100*}  & \textbf{94.1} & \textbf{75.5} & \textbf{89.9} & \textbf{87.4}    \\ \cline{2-9} 
                                                                & VMIFGSM & 93.7          & 92.7          & 99.8*          & 96.1          & 85.7          & 94.0          & 92.4             \\
                                                                & VNIFGSM & 93.9          & 92.5          & 99.8*          & 96.7          & 85.6          & 95.2          & 92.8             \\
                                                                & VDTA    & 96.8          & \textbf{96.0} & \textbf{99.9*} & \textbf{98.8} & 91.5          & 97.3          & \textbf{96.1}    \\
                                                                & VDTA+NP & \textbf{97.2} & 95.6          & 99.6*          & 98.1          & \textbf{92.2} & \textbf{97.5} & \textbf{96.1}    \\ \hline
\end{tabular}
\end{center}
\end{table}
\begin{table}[t]
\caption{The ASR (\%) comparison on six models on ImageNet~\cite{ImageNet}. The gradient-based attacks are enhanced by SGM~\cite{SGM}, respectively. The adversarial examples are generated by the surrogate model, ResNet50~\cite{ResNet}. * denotes the ASR under the white-box setting. \textbf{Average} means the average value except *.}
\label{tab:sgm}
\scriptsize
\begin{center}
\begin{tabular}{cccccccc}
\hline
Attack      & VGG16         & VGG19         & Res50         & Res152        & Inc-v3        & Mob-v2        & \textbf{Average} \\ \hline
MIFGSM+SGM  & 92.6          & 89.7          & 99.9*          & 90.1          & 68.1          & 90.5          & 88.5             \\
NIFGSM+SGM  & 93.7          & 92            & 99.9*          & 93.4          & 70.8          & 92.6          & 90.4             \\
DTA+SGM     & 96.6          & 95.6          & \textbf{100*}  & 97.8          & 81.1          & 97            & 94.7             \\
DTA+NP+SGM  & \textbf{98}   & \textbf{97.4} & \textbf{100*}  & \textbf{98.1} & \textbf{82.3} & \textbf{97.8} & \textbf{95.6}    \\ \hline
VMIFGSM+SGM & 94.6          & 93.9          & \textbf{99.9*} & 93.4          & 77.4          & 93.6          & 92.1             \\
VNIFGSM+SGM & 95.9          & 94.7          & \textbf{99.9*} & 93.5          & 78.2          & 95.4          & 92.9             \\
VDTA+SGM    & \textbf{96.9} & \textbf{95.5} & \textbf{99.9*} & \textbf{95.8} & \textbf{85}   & \textbf{96.6} & \textbf{95.0}    \\
VDTA+NP+SGM & 95.1          & 94.0          & 98.5*          & 93.6          & 82.3          & 95.1          & 93.1             \\ \hline
\end{tabular}
\end{center}
\end{table}
\begin{table}[t]
\caption{The ASR (\%) comparison on six models on ImageNet~\cite{ImageNet}. The gradient-based attacks are enhanced by More Bayesian (MB)~\cite{MoreBayesian}, respectively. The adversarial examples are generated by the surrogate model, ResNet50~\cite{ResNet}. * denotes the ASR under the white-box setting. \textbf{Average} means the average value except *.}
\label{tab:mb}
\scriptsize
\begin{center}
\begin{tabular}{cccccccc}
\hline
Attack    & VGG16         & VGG19         & Res50        & Res152        & Inc-v3        & Mob-v2        & \textbf{Average} \\ \hline
MB        & 97.6          & 97.4          & \textbf{100} & 99.6          & 76.4          & 97.9          & 94.8             \\
NIFGSM+MB & 82.6          & 77.7          & \textbf{100} & 81.9          & 46.9          & 79.8          & 78.2             \\
DTA+MB    & \textbf{98.4} & \textbf{98.9} & \textbf{100} & \textbf{99.8} & \textbf{89.5} & \textbf{99.6} & \textbf{97.7}    \\ \hline
\end{tabular}
\end{center}
\end{table}

\subsection{Compatibility of our methods to Other Types of Attacks}
\label{sec:combination-fia-and-input-transformation}
In this subsection, we combine our methods with a feature importance-aware attack (i.e., FIA~\cite{FIA}), three input transformation methods (i.e., DIM~\cite{DI-FGSM}, TIM~\cite{TI-FGSM} and SIM~\cite{SI-NI-FGSM}), SGM~\cite{SGM} and More Bayesian (MB)~\cite{MoreBayesian} to verify the compatibility of our methods (including both direction tuning attack and network pruning method).

Table~\ref{tab:fia} shows the ASRs on six victim models when the gradient-based attacks are combined with the feature importance-aware attack (i.e., FIA~\cite{FIA}). The results show that our DTA and VDTA can assist FIA to achieve the greatest ASR in comparison with other gradient-based attacks (i.e., MI/NIFGSM and their variance tuning-based version). Besides, our DTA also spends less time consumption in comparison with VMI/VNIFGSM, in which the gradient of input in our method is computed 100 times (i.e., $K\cdot T=10\times 10$), while that in variance tuning-based gradient attacks is 200 times (i.e., $N\cdot T=20\times 10$).

Table~\ref{tab:input-transformation-methods} shows the ASRs on six victim models when our methods, including DTA, VDTA, and NP, are combined with three input transformation methods, respectively. The results show that direction tuning attacks are perfectly compatible with each input transformation method, and the network pruning method can further improve the transferability of adversarial examples. In particular, the combination of both DTA+NP and DIM~\cite{DI-FGSM} not only has the highest transferability but also has less time consumption in comparison with variance tuning-based gradient attacks. Note that, in this evaluation, VDTA computes the gradient of input 2000 times (i.e., $K\cdot N\cdot T=10\times 20\times 10$).

Table~\ref{tab:sgm} shows that our methods also have better compatibility with the skip gradient method~\cite{SGM} in comparison with the other gradient-based attacks. When compared with MI/NI-FGSM~\cite{MI-FGSM,SI-NI-FGSM}, our DTA can assist SGM to improve the average ASR from 90.4\% to 94.7\% and the network pruning method can assist SGM to further improve the transferability by 0.9\%. When compared with VMI/VNIFGSM~\cite{VMI-FGSM}, our VDTA can assist SGM to improve the average ASR from 92.9\% to 95.0\%, but the network pruning method can not further improve the transferability of our VDTA due to over-pruning. Additionally, the combination of DTA and network pruning method not only have better compatibility with SGM but also spend less time generating adversarial examples when compared with the variance tuning-based gradient attacks.

Table~\ref{tab:mb} shows that our DTA can further enhance the transferability of the latest model diversification method, i.e. More Bayesian (MB)~\cite{MoreBayesian}. To align with MB~\cite{MoreBayesian}, the parameters are set as $\epsilon=8/255$, $T=50$, $\alpha=1/255$ and the rescaling factor MB is set to $1.5$. The results demonstrate that our DTA can effectively enhance the transferability of More Bayesian~\cite{MoreBayesian} from 94.8\% to 97.7\%. Additionally, NIFGSM significantly decreases the transferability of More Bayesian~\cite{MoreBayesian} due to bad convergence.

\begin{figure}[ht]
\begin{center}
\centerline{\includegraphics[width=\textwidth]{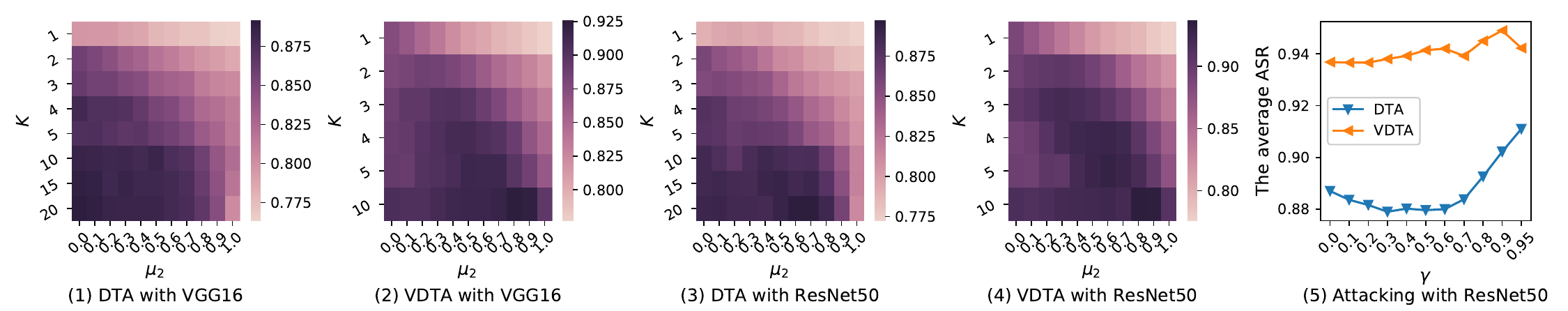}}
\caption{The average ASR (\%) on five models on ImageNet~\cite{ImageNet} dataset under the black-box setting when varying factor $K$ and $\mu_2$ in direction tuning attacks, and $\gamma$ in network pruning method. The adversarial examples are crafted with the surrogate model, VGG16 or ResNet50. Sub-figures (1) to (4) show the sensitivities of $K$ and $\mu_2$ in DTA or VDTA, and Sub-figure (5) represents that of $\gamma$ in network pruning method.}
\label{fig:ablation_k_decay2_gamma}
\end{center}
\end{figure}

\subsection{The Sensitivity Analysis}
\label{sec:sensitivity-analysis}
To analyze the sensitivity of the parameters $K$ and $\mu_2$ in direction tuning attacks, and $\gamma$ in network pruning method, we explore the impact of parameters $K$, $\mu_2$ and $\gamma$ on the average ASR in Fig.~\ref{fig:ablation_k_decay2_gamma} when $K\in[1,20]$, $\mu_2\in[0,1]$ and $\gamma\in$[0,1].

\textbf{The parameters $K$ and $\mu_2$ in direction tuning attacks.} As depicted in Fig.~\ref{fig:ablation_k_decay2_gamma}-(1) to (4), our DTA and VDTA attacks have higher transferability than NI-FGSM~\cite{SI-NI-FGSM} and VNI-FGSM~\cite{VMI-FGSM} in the most combination of $K$ and $\mu_2$. Besides, the transferability of our attacks becomes greater as the increase of $K$. When combining a moderate decay $\mu_2$, our attacks can achieve the best transferability in different sizes of decay $\mu_2$. Particularly, even if $K$ is determined as a small number (i.e., $K=2$), the transferability can be still significantly increased. Note that, when $K=1$ and $\mu_2=0$, DTA and VDTA will be degraded to NI-FGSM~\cite{SI-NI-FGSM} and VNI-FGSM~\cite{VMI-FGSM}, respectively. Due to the time complexity of our attacks increases as the parameter $K$ increases, a suitable $K$ is required to balance the transferability and computation cost. Eventually, we set the parameters as $K=10$ and $\mu_2=0.0$ on DTA, and $K=10$ and $\mu_2=0.8$ on VDTA, respectively.

\textbf{The parameter $\gamma$ in network pruning method.} We then investigate the influence of the parameter $\gamma$ on the transferability of direction tuning attacks where the parameters $K$ and $\mu_2$ are fixed as the optimal values. When $\gamma=0$, the network pruning method will be invalid. As depicted in Fig.~\ref{fig:ablation_k_decay2_gamma}-(5), when $\gamma\in [0,0.9]$, the average ASR gradually rises as $\gamma$ increases, which demonstrates the effectiveness of network pruning method on enhancing the transferability of our direction tuning-based attacks. Therefore, the parameter $\gamma$ of the network pruning method is set as 0.9 when the method is applied in our direction tuning-based attacks.

%
\begin{table}[t]
\caption{The average ASR (\%) comparison on five models on ImageNet~\cite{ImageNet} under the black-box setting when the time complexity of each gradient-based attacks is equal. The adversarial examples are generated by the surrogate model, VGG16 or ResNet50. Note that, to keep the same of the time complexity with DTA or VDTA, we increase the number of iterations of other attacks by $K$ times. TC denotes time complexity.}
\label{tab:imagenet-the-same-of-time-complexity}
\scriptsize
\begin{center}
\begin{tabular}{cccccc}
\hline
Model                                                                 & Attack   & TC                     & ASR  & TC                      & ASR           \\ \hline
\multirow{7}{*}{VGG16}        & MI-FGSM  & \multirow{3}{*}{$O(T)$}  & 68.2 & \multirow{3}{*}{$O(K\cdot T)$}  & 67.1          \\
                              & NI-FGSM  &                        & 70.8 &                         & 69.5          \\
                              & DTA      &                        & -    &                         & \textbf{82.9} \\ \cline{2-6} 
                              & PGD      & \multirow{4}{*}{$O(N\cdot T)$} & -    & \multirow{4}{*}{$O(K\cdot N\cdot T)$} & 45.9          \\
                              & VMI-FGSM &                        & 80.0 &                         & 86.3          \\
                              & VNI-FGSM &                        & 81.9 &                         & 85.8          \\
                              & VDTA     &                        & -    &                         & \textbf{88.9} \\ \hline
\multirow{7}{*}{ResNet50}     & MI-FGSM  & \multirow{3}{*}{$O(T)$}  & 70.5 & \multirow{3}{*}{$O(K\cdot T)$}  & 65.0          \\
                              & NI-FGSM  &                        & 75.7 &                         & 70.0          \\
                              & DTA      &                        & -    &                         & \textbf{87.7} \\ \cline{2-6} 
                              & PGD      & \multirow{4}{*}{$O(N\cdot T)$} & -    & \multirow{4}{*}{$O(K\cdot N\cdot T)$} & 27.4          \\
                              & VMI-FGSM &                        & 84.5 &                         & 89.5          \\
                              & VNI-FGSM &                        & 87.9 &                         & 90.2          \\
                              & VDTA     &                        & -    &                         & \textbf{93.3} \\ \hline
\multirow{7}{*}{Inception-v3} & MI-FGSM  & \multirow{3}{*}{$O(T)$}  & 39.9 & \multirow{3}{*}{$O(K\cdot T)$}  & 27.1          \\
                              & NI-FGSM  &                        & 47.4 &                         & 34.3          \\
                              & DTA      &                        & -    &                         & \textbf{56.6} \\ \cline{2-6} 
                              & PGD      & \multirow{4}{*}{$O(N\cdot T)$} & -    & \multirow{4}{*}{$O(K\cdot N\cdot T)$} & 9.0           \\
                              & VMI-FGSM &                        & 55.8 &                         & 56.1          \\
                              & VNI-FGSM &                        & 62.1 &                         & 59.9          \\
                              & VDTA     &                        & -    &                         & \textbf{67.9} \\ \hline
\end{tabular}
\end{center}
\end{table}

\subsection{Time Complexity Analysis}
\label{sec:the-same-of-time-complexity}
Direction tuning attacks, which add an inner loop in each step update, increase the time complexity in comparison with the current transferable attacks. For example, the time complexity of I/MI/NI-FGSM and VMI/VNI-FGSM is $O(T)$ and $O(N\cdot T)$, respectively, but that of DTA and VDTA is $O(K\cdot T)$ and $O(K\cdot N\cdot T)$. Due to the small size of the optimal parameter $K$ (i.e. $K=10$), DTA computes the gradient of input 100 times (i.e., $K\cdot T=10\times 10$) and VDTA computes 2000 times (i.e., $K\cdot N \cdot T=10\times 20\times 10$). Although the computation cost of VDTA is an order of magnitude higher than that of VMI/VNI-FGSM, it is acceptable like PGD-1000~\cite{Madry-AT}.


More importantly, our method can effectively improve the transferability of the gradient-based attacks at the same time complexity. To verify this, we increase the number of iterations of MI/NI-FGSM and VMI/VNI-FGSM by $K$ times. As shown in Table~\ref{tab:imagenet-the-same-of-time-complexity}, we found that i) the transferability of MI/NI-FGSM is decreased as the number of iterations increases. ii) Due to the stabilization property of variance tuning, in most cases, increasing iterations appropriately can improve the transferability of adversarial examples, e.g., VMI/VNI-FGSM with VGG16~\cite{VGG} and ResNet50~\cite{ResNet} as the surrogate models. iii) However, our VDTA still has significantly higher transferability under the same time complexity in comparison with other attacks under different networks~\cite{ResNet,VGG,Inception-v3} as the surrogate models.


\section{Conclusion}
\label{sec:conclusion}
In this paper, the direction tuning attacks and network pruning method are proposed to enhance the transferability of adversarial examples. Specifically, the direction tuning attack embeds the small sampling step length in each large update step to enhance the transferability of the generated adversarial examples by decreasing the angle between the actual update direction and the steepest update direction and reducing update oscillation. The network pruning method is proposed to smooth the classification boundary to further eliminate the update oscillation by stabilizing the update direction. Extended experiments on the ImageNet dataset verified that our methods can significantly improve the transferability of adversarial examples with and without defenses. In addition, the experimental results demonstrated our approach is compatible with other types of attacks well. Moreover, our method enables good transferability by reasonably increasing time consumption. Finally, we believe our study is beneficial to the adversarial defense community to design robust defense methods against attacks with strong transferability.

\appendix

\section{The Pseudocode of VDTA}
\label{appendix:pseudocode-vdta}
The complete execution pseudocode of VDTA is shown in Algorithm~\ref{alg:vdta}.


\bibliographystyle{unsrt} 
\bibliography{ref}

\begin{thebibliography}{10}

\bibitem{ResNet}
Kaiming He, Xiangyu Zhang, Shaoqing Ren, and Jian Sun.
\newblock Deep residual learning for image recognition.
\newblock In {\em {CVPR}}, pages 770--778. {IEEE} Computer Society, 2016.

\bibitem{Intriguing-properties}
Christian Szegedy, Wojciech Zaremba, Ilya Sutskever, Joan Bruna, Dumitru Erhan,
  Ian~J. Goodfellow, and Rob Fergus.
\newblock Intriguing properties of neural networks.
\newblock In {\em {ICLR} (Poster)}, 2014.

\bibitem{VGG}
Karen Simonyan and Andrew Zisserman.
\newblock Very deep convolutional networks for large-scale image recognition.
\newblock In {\em {ICLR}}, 2015.

\bibitem{Inception-v3}
Christian Szegedy, Vincent Vanhoucke, Sergey Ioffe, Jonathon Shlens, and
  Zbigniew Wojna.
\newblock Rethinking the inception architecture for computer vision.
\newblock In {\em {CVPR}}, pages 2818--2826. {IEEE} Computer Society, 2016.

\bibitem{CW}
Nicholas Carlini and David~A. Wagner.
\newblock Towards evaluating the robustness of neural networks.
\newblock In {\em {IEEE} Symposium on Security and Privacy}, pages 39--57.
  {IEEE} Computer Society, 2017.

\bibitem{Madry-AT}
Aleksander Madry, Aleksandar Makelov, Ludwig Schmidt, Dimitris Tsipras, and
  Adrian Vladu.
\newblock Towards deep learning models resistant to adversarial attacks.
\newblock In {\em {ICLR} (Poster)}. OpenReview.net, 2018.

\bibitem{FGSM}
Ian~J. Goodfellow, Jonathon Shlens, and Christian Szegedy.
\newblock Explaining and harnessing adversarial examples.
\newblock In {\em {ICLR} (Poster)}, 2015.

\bibitem{I-FGSM}
Alexey Kurakin, Ian~J. Goodfellow, and Samy Bengio.
\newblock Adversarial examples in the physical world.
\newblock In {\em {ICLR} (Workshop)}. OpenReview.net, 2017.

\bibitem{MI-FGSM}
Yinpeng Dong, Fangzhou Liao, Tianyu Pang, Hang Su, Jun Zhu, Xiaolin Hu, and
  Jianguo Li.
\newblock Boosting adversarial attacks with momentum.
\newblock In {\em {CVPR}}, pages 9185--9193. Computer Vision Foundation /
  {IEEE} Computer Society, 2018.

\bibitem{SI-NI-FGSM}
Jiadong Lin, Chuanbiao Song, Kun He, Liwei Wang, and John~E. Hopcroft.
\newblock Nesterov accelerated gradient and scale invariance for adversarial
  attacks.
\newblock In {\em {ICLR}}. OpenReview.net, 2020.

\bibitem{VMI-FGSM}
Xiaosen Wang and Kun He.
\newblock Enhancing the transferability of adversarial attacks through variance
  tuning.
\newblock In {\em {CVPR}}, pages 1924--1933. Computer Vision Foundation /
  {IEEE}, 2021.

\bibitem{Square}
Maksym Andriushchenko, Francesco Croce, Nicolas Flammarion, and Matthias Hein.
\newblock Square attack: {A} query-efficient black-box adversarial attack via
  random search.
\newblock In {\em {ECCV} {(23)}}, volume 12368 of {\em Lecture Notes in
  Computer Science}, pages 484--501. Springer, 2020.

\bibitem{ZOO}
Pin{-}Yu Chen, Huan Zhang, Yash Sharma, Jinfeng Yi, and Cho{-}Jui Hsieh.
\newblock {ZOO:} zeroth order optimization based black-box attacks to deep
  neural networks without training substitute models.
\newblock In {\em AISec@CCS}, pages 15--26. {ACM}, 2017.

\bibitem{DI-FGSM}
Cihang Xie, Zhishuai Zhang, Yuyin Zhou, Song Bai, Jianyu Wang, Zhou Ren, and
  Alan~L. Yuille.
\newblock Improving transferability of adversarial examples with input
  diversity.
\newblock In {\em {CVPR}}, pages 2730--2739. Computer Vision Foundation /
  {IEEE}, 2019.

\bibitem{TI-FGSM}
Yinpeng Dong, Tianyu Pang, Hang Su, and Jun Zhu.
\newblock Evading defenses to transferable adversarial examples by
  translation-invariant attacks.
\newblock In {\em {CVPR}}, pages 4312--4321. Computer Vision Foundation /
  {IEEE}, 2019.

\bibitem{Admix}
Xiaosen Wang, Xuanran He, Jingdong Wang, and Kun He.
\newblock Admix: Enhancing the transferability of adversarial attacks.
\newblock In {\em {ICCV}}, pages 16138--16147. {IEEE}, 2021.

\bibitem{FIA}
Zhibo Wang, Hengchang Guo, Zhifei Zhang, Wenxin Liu, Zhan Qin, and Kui Ren.
\newblock Feature importance-aware transferable adversarial attacks.
\newblock In {\em {ICCV}}, pages 7619--7628. {IEEE}, 2021.

\bibitem{NAA}
Jianping Zhang, Weibin Wu, Jen{-}tse Huang, Yizhan Huang, Wenxuan Wang, Yuxin
  Su, and Michael~R. Lyu.
\newblock Improving adversarial transferability via neuron attribution-based
  attacks.
\newblock {\em CoRR}, abs/2204.00008, 2022.

\bibitem{SGM}
Dongxian Wu, Yisen Wang, Shu{-}Tao Xia, James Bailey, and Xingjun Ma.
\newblock Skip connections matter: On the transferability of adversarial
  examples generated with resnets.
\newblock In {\em {ICLR}}. OpenReview.net, 2020.

\bibitem{LinBP}
Yiwen Guo, Qizhang Li, and Hao Chen.
\newblock Backpropagating linearly improves transferability of adversarial
  examples.
\newblock In {\em NeurIPS}, 2020.

\bibitem{LGV}
Martin Gubri, Maxime Cordy, Mike Papadakis, Yves~Le Traon, and Koushik Sen.
\newblock {LGV:} boosting adversarial example transferability from large
  geometric vicinity.
\newblock In {\em {ECCV} {(4)}}, volume 13664 of {\em Lecture Notes in Computer
  Science}, pages 603--618. Springer, 2022.

\bibitem{MoreBayesian}
Qizhang Li, Yiwen Guo, Wangmeng Zuo, and Hao Chen.
\newblock Making substitute models more bayesian can enhance transferability of
  adversarial examples.
\newblock {\em CoRR}, abs/2302.05086, 2023.

\bibitem{RAP}
Zeyu Qin, Yanbo Fan, Yi~Liu, Li~Shen, Yong Zhang, Jue Wang, and Baoyuan Wu.
\newblock Boosting the transferability of adversarial attacks with reverse
  adversarial perturbation.
\newblock {\em CoRR}, abs/2210.05968, 2022.

\bibitem{Self-ensemble}
Muzammal Naseer, Kanchana Ranasinghe, Salman Khan, Fahad~Shahbaz Khan, and
  Fatih Porikli.
\newblock On improving adversarial transferability of vision transformers.
\newblock In {\em {ICLR}}. OpenReview.net, 2022.

\bibitem{Dynamic-Cues}
Muzammal Naseer, Ahmad Mahmood, Salman~H. Khan, and Fahad~Shahbaz Khan.
\newblock Boosting adversarial transferability using dynamic cues.
\newblock {\em CoRR}, abs/2302.12252, 2023.

\bibitem{gradient-alignment}
Ambra Demontis, Marco Melis, Maura Pintor, Matthew Jagielski, Battista Biggio,
  Alina Oprea, Cristina Nita{-}Rotaru, and Fabio Roli.
\newblock Why do adversarial attacks transfer? explaining transferability of
  evasion and poisoning attacks.
\newblock In {\em {USENIX} Security Symposium}, pages 321--338. {USENIX}
  Association, 2019.

\bibitem{Fast-AT}
Eric Wong, Leslie Rice, and J.~Zico Kolter.
\newblock Fast is better than free: Revisiting adversarial training.
\newblock In {\em {ICLR}}. OpenReview.net, 2020.

\bibitem{Free-AT}
Ali Shafahi, Mahyar Najibi, Amin Ghiasi, Zheng Xu, John~P. Dickerson, Christoph
  Studer, Larry~S. Davis, Gavin Taylor, and Tom Goldstein.
\newblock Adversarial training for free!
\newblock In {\em NeurIPS}, pages 3353--3364, 2019.

\bibitem{Trades}
Hongyang Zhang, Yaodong Yu, Jiantao Jiao, Eric~P. Xing, Laurent~El Ghaoui, and
  Michael~I. Jordan.
\newblock Theoretically principled trade-off between robustness and accuracy.
\newblock In {\em {ICML}}, volume~97 of {\em Proceedings of Machine Learning
  Research}, pages 7472--7482. {PMLR}, 2019.

\bibitem{Ensemble-Adversarial-Training}
Florian Tram{\`{e}}r, Alexey Kurakin, Nicolas Papernot, Ian~J. Goodfellow, Dan
  Boneh, and Patrick~D. McDaniel.
\newblock Ensemble adversarial training: Attacks and defenses.
\newblock In {\em {ICLR} (Poster)}. OpenReview.net, 2018.

\bibitem{Confidence-Threshold-Reduction}
Xiangyuan Yang, Jie Lin, Han Zhang, Xinyu Yang, and Peng Zhao.
\newblock Improving the robustness and generalization of deep neural network
  with confidence threshold reduction.
\newblock 2022.

\bibitem{RP}
Cihang Xie, Jianyu Wang, Zhishuai Zhang, Zhou Ren, and Alan~L. Yuille.
\newblock Mitigating adversarial effects through randomization.
\newblock In {\em {ICLR} (Poster)}. OpenReview.net, 2018.

\bibitem{Bit-Red}
Weilin Xu, David Evans, and Yanjun Qi.
\newblock Feature squeezing: Detecting adversarial examples in deep neural
  networks.
\newblock In {\em {NDSS}}. The Internet Society, 2018.

\bibitem{JPEG}
Chuan Guo, Mayank Rana, Moustapha Ciss{\'{e}}, and Laurens van~der Maaten.
\newblock Countering adversarial images using input transformations.
\newblock In {\em {ICLR} (Poster)}. OpenReview.net, 2018.

\bibitem{RS}
Jeremy~M. Cohen, Elan Rosenfeld, and J.~Zico Kolter.
\newblock Certified adversarial robustness via randomized smoothing.
\newblock In {\em {ICML}}, volume~97 of {\em Proceedings of Machine Learning
  Research}, pages 1310--1320. {PMLR}, 2019.

\bibitem{FD}
Zihao Liu, Qi~Liu, Tao Liu, Nuo Xu, Xue Lin, Yanzhi Wang, and Wujie Wen.
\newblock Feature distillation: Dnn-oriented {JPEG} compression against
  adversarial examples.
\newblock In {\em {CVPR}}, pages 860--868. Computer Vision Foundation / {IEEE},
  2019.

\bibitem{NRP}
Muzammal Naseer, Salman~H. Khan, Munawar Hayat, Fahad~Shahbaz Khan, and Fatih
  Porikli.
\newblock A self-supervised approach for adversarial robustness.
\newblock In {\em {CVPR}}, pages 259--268. Computer Vision Foundation / {IEEE},
  2020.

\bibitem{ImageNet}
Olga Russakovsky, Jia Deng, Hao Su, Jonathan Krause, Sanjeev Satheesh, Sean Ma,
  Zhiheng Huang, Andrej Karpathy, Aditya Khosla, Michael~S. Bernstein,
  Alexander~C. Berg, and Li~Fei{-}Fei.
\newblock Imagenet large scale visual recognition challenge.
\newblock {\em Int. J. Comput. Vis.}, 115(3):211--252, 2015.

\bibitem{Network-pruning}
Sunil Vadera and Salem Ameen.
\newblock Methods for pruning deep neural networks.
\newblock {\em {IEEE} Access}, 10:63280--63300, 2022.

\bibitem{SAP}
Guneet~S. Dhillon, Kamyar Azizzadenesheli, Zachary~C. Lipton, Jeremy Bernstein,
  Jean Kossaifi, Aran Khanna, and Animashree Anandkumar.
\newblock Stochastic activation pruning for robust adversarial defense.
\newblock In {\em {ICLR} (Poster)}. OpenReview.net, 2018.

\bibitem{CIFAR}
Alex Krizhevsky.
\newblock Learning multiple layers of features from tiny images.
\newblock 2009.

\bibitem{MobileNet}
Mark Sandler, Andrew~G. Howard, Menglong Zhu, Andrey Zhmoginov, and
  Liang{-}Chieh Chen.
\newblock Mobilenetv2: Inverted residuals and linear bottlenecks.
\newblock In {\em {CVPR}}, pages 4510--4520. Computer Vision Foundation /
  {IEEE} Computer Society, 2018.

\bibitem{pytorch-image-models}
Ross Wightman.
\newblock Pytorch image models.
\newblock \url{https://github.com/rwightman/pytorch-image-models}, 2019.

\bibitem{torchvision-models}
Howard Huang.
\newblock torchvision.models.
\newblock \url{https://pytorch.org/vision/stable/models.html}, 2017.

\end{thebibliography}

\begin{algorithm}[!b]
	\caption{VDTA}
	\label{alg:vdta}
	\begin{algorithmic}[1]
		\Require
		The natural example $x$ with its ground truth label $y$; the surrogate model $f$; the loss function $L$; The magnitude of perturbation $\epsilon$; the number of iteration $T$; the decay factor $\mu_1$; The factor $\beta$ for the upper bound of neighborhood and number of example $N$ for variance tuning;
		\Require
		The number of iteration $K$ and decay factor $\mu_2$ in the inner loop of our \textit{direction tuning}.
		\Ensure
		An adversarial example $x^{adv}$.
		\State $\alpha=\epsilon/T$
		\State $g_0=0; v_0=0; x^{adv}_0=x$
		\For {$t=0\rightarrow T-1$}
		\State $g_{t,0}=g_t;v_{t,0}=v_t;x^{adv}_{t,0}=x^{adv}_t$
		\For {$k=0\rightarrow K-1$}
		\State Get $x^{nes}_{t,k}$ by $x^{nes}_{t,k}=x^{adv}_{t,k}+\alpha\cdot\mu_1\cdot g_{t,k}$ (i.e., Eq.~\ref{equation:looking-ahead-dta})
		\State Get the gradient $g_{t,k+1}$ by Eq.~\ref{equation:vdta-inner-gradient} 
		\State Get the variance $v_{t,k+1}$ by Eq.~\ref{equation:vdta-variance-tuning}
		\State Update $x^{adv}_{t,k+1}$ by Eq.~\ref{equation:dta-update}
		\EndFor
		\State Update $v_{t+1}$ by $v_{t+1}=v_{t,1}$ \Comment{$v_{t,1}$ is equal to $v_{t+1}$ calculated by Eq.~\ref{equation:vmifgsm-variance}}
		\State Update $g_{t+1}$ by Eq.~\ref{equation:dta-gradient}
		\State Update $x^{adv}_{t+1}$ by Eq.~\ref{equation:mifgsm-update}
		\EndFor
		\State $x^{adv}=x^{adv}_T$\\
		\Return $x^{adv}$
	\end{algorithmic}
\end{algorithm}
%
\end{document}